\documentclass[a4paper,11pt]{article}
\usepackage[parfill]{parskip}    
\usepackage[paper=a4paper,dvips,top=1in,left=1.5in,right=1.5in,
    foot=1cm,bottom=1in]{geometry}
\usepackage[english]{babel}
\usepackage[latin1]{inputenc}
\usepackage{latexsym}
\usepackage{epsfig}
\usepackage{graphicx}
\usepackage{color}
\usepackage{amsfonts}
\usepackage{bbm}
\usepackage{amsmath}
\usepackage{amssymb}
\usepackage{amsthm}
\usepackage{epstopdf}
\usepackage{subfigure}
\usepackage{enumerate}
\RequirePackage{xspace}
\usepackage[ruled, section]{algorithm}
\usepackage{algpseudocode}

\newtheorem{theorem}{Theorem}[section]
\newtheorem{lemma}[theorem]{Lemma}

\newtheorem{definition}[theorem]{Definition}
\newtheorem{corollary}[theorem]{Corollary}

\newcommand{\E}{ \mathbb E}
\newcommand{ \R}{ \mathbb R}
\renewcommand{\Pr}{ \mathrm P}
\newcommand{ \D}{ \mathbb D}
\newcommand{ \s}{ \mathcal S }

\newcommand{ \cb}{ \mathcal B}
\newcommand{ \cd}{ \mathcal D}

\newcommand{ \cV}{ \mathcal V}

\newcommand{\ga}{\alpha}
\newcommand{\gb}{\beta}
\newcommand{\gd}{\delta}
\newcommand{\gre}{\epsilon}
\newcommand{\gve}{\varepsilon}

\newcommand{\gt}{\theta}
\newcommand{\gl}{\lambda}
\newcommand{\gs}{\sigma}
\newcommand{\om}{\omega}

\newcommand{\gm}{\gamma}
\newcommand{\G}{\Gamma}

\newcommand{\gkp}{\kappa}

\newcommand{\bh}{\mathbf{H}}
\newcommand{\bj}{\mathbf{J}}
\newcommand{\bw}{\mathbf{W}}
\newcommand{\bld}{\mathbf{D}}
\newcommand{\br}{\mathbf{R}}
\newcommand{\bc}{\mathbf{C}}

\newcommand{\rar}{ \rightarrow}

\newcommand{\V}[1]{\ensuremath{\boldsymbol{#1}}\xspace}
\newcommand{\til}[1]{\ensuremath{\tilde{#1}}}

\title{\sc{Community Detection in Networks using Graph Distance}}
\author{Sharmodeep Bhattacharyya and Peter J. Bickel}
\date{\today}

\begin{document}
\maketitle

%
%
%
%
%
%
%
%
%

\begin{abstract}
The study of networks has received increased  attention recently not only from the social sciences and statistics but also from physicists, computer scientists and mathematicians. One of the principal problem in networks is community detection. Many algorithms have been proposed for community finding \cite{mcsherry2001spectral} \cite{rohe2011spectral} but most of them do not have have theoretical guarantee for sparse networks and networks close to the phase transition boundary proposed by physicists \cite{decelle2011asymptotic}. There are some exceptions but all have incomplete theoretical basis \cite{coja2009finding} \cite{chen2012fitting} \cite{krzakala2013spectral}. Here we propose an algorithm based on the graph distance of vertices in the network. We give theoretical guarantees that our method works in identifying communities for block models, degree-corrected block models \cite{karrer2011stochastic} and block models with the number of communities growing with number of vertices. Despite favorable simulation results, we are not yet able to conclude that our method is satisfactory for worst possible case. We illustrate on a network of political blogs, Facebook networks and some other networks. 
\end{abstract}

\section{Introduction}
\label{intro}
The study of networks has received increased  attention recently  not only from the social sciences and statistics but also from physicists, computer scientists and mathematicians. With the information boom, a huge number of network data sets have come into prominence. In biology - gene transcription networks, protein-protein interaction network, in social media - Facebook, Twitter, Linkedin networks, information networks arising in connection with text mining, technological networks such as the Internet, ecological and epidemiological networks and many others have appeared. Although the study of networks has a long history in physics, social sciences and mathematics literature and informal methods of analysis have arisen in many fields of application, statistical inference on network models as opposed to descriptive statistics, empirical modeling and some Bayesian approaches \cite{newman2009networks} \cite{kolaczyk2009statistical} \cite{hoff2002latent} has not been addressed extensively in the literature. A mathematical and systematic study of statistical inference on network models has only started in recent years.

One of the fundamental questions in analysis of such data is detecting and modeling community structure within the network. A lot of algorithmic approaches to community detection have been proposed, particularly in the physics and computer science literature \cite{newman2006modularity} \cite{leskovec2010empirical} \cite{fortunato2010community}. In terms of community detection, there are two different goals that researchers have tried to pursue - 
\begin{itemize}
\item \textbf{Algorithmic Goal:} Identify the community each vertex of the network belongs to. \\
\item \textbf{Theoretical Goal:} If the network is generated by an underlying generative model, then, what is the probability of success for the algorithm.
\end{itemize}

\subsection{Algorithms}
\label{sec_alg_goal}
Several popular algorithms for community detection have been proposed in physics, computer science and statistics literature. Most of these algorithms show decent performance in community detection for selected real-world and simulated networks \cite{lancichinetti2009community} and have polynomial time complexity. We shall briefly mention some of these algorithms.
\begin{enumerate}
\item Modularity maximizing methods \cite{newman2004finding}. One of the most popular method of community detection. The problem is NP hard but spectral relaxations of polynomial complexity exist \cite{newman2006finding}.
\item Hierarchical clustering techniques \cite{clauset2004finding}. 
\item Spectral clustering based methods \cite{mcsherry2001spectral} \cite{coja2009finding}, \cite{rohe2011spectral} \cite{chaudhuri2012spectral}. These methods are also very popular. Most of the time these methods have linear or polynomial running times. Mostly shown to work for dense graphs only.
\item Profile likelihood maximization \cite{bickel2009nonparametric}. The problem is NP hard, but heuristic algorithms have been proposed, which have good performance for dense graphs.
\item Stochastic Model based methods:
\begin{itemize}
\item MCMC based likelihood maximization by Gibbs Sampling, the cavity method and belief propagation based on stochastic block model. \cite{decelle2011asymptotic}
\item Variational Likelihood Maximization based on stochastic block model \cite{celisse2011consistency}, \cite{bickel2012asymptotic}. Polynomial running time but appears to work only for dense graphs.
\item Pseudo-likelihood Maximization \cite{chen2012fitting}. Fast method which works well for both dense and sparse graphs. But the method is not fully justified.
\item Model-based:
\begin{itemize}
\item[(a)] Mixed Membership Block Model \cite{airoldi2008mixed}. Iterative method and works for dense graphs. The algorithm for this model is based on variational approximation of the maximum likelihood estimation.
\item[(b)] Degree-corrected block model \cite{karrer2011stochastic}: Incorporates degree inhomogeneity in the model. Algorithms based on maximum likelihood and profile likelihood estimation has been developed.
\item[(c)] Overlapping stochastic block model \cite{latouche2011overlapping}: Stochastic block model where each vertex can lie within more than one community. The algorithm for this model is based on variational approximation of the maximum likelihood estimation.
\item[(d)] Mixed configurations model \cite{ball2011efficient}: Another extension to degree-corrected stochastic block model, where, the model is a mixture of configurations model (degree-corrected block model with one block) and each vertex can lie in more than one community. The algorithm for this model is based on the EM algorithm and maximum likelihood estimation.
\end{itemize}
\end{itemize}
\item Model based clustering \cite{handcock2007model}.
\end{enumerate}

\subsection{Theoretical Goal}
\label{sec_th_goal}
The stochastic block model (SBM) is perhaps the most commonly used and best studied model for community detection. An SBM with $Q$ blocks states that each node belongs to a community $\V{c} = (c_1, \ldots ,c_n) \in \{1, \ldots,Q\}$ which are drawn independently from the multinomial distribution with parameter $\V{\pi} = (\pi_1, \ldots, \pi_Q)$, where $\pi_i > 0$ for all $i$, and $Q$ is the number of communities, assumed known. Conditional on the labels, the edge variables $A_{ij}$ for $i < j$ are independent Bernoulli variables with
\begin{align}
\label{eq_sbm_adj}
\E[A_{ij}|\V{c}] = P_{c_ic_j}  ,
\end{align}
where $P = [P_{ab}]$ and $K = [K_{ab}]$ are $Q \times Q$ symmetric matrix. $P$ can be considered the \textit{connection probability} matrix, where as $K$ is the \textit{kernel} matrix for the connection. So, we have $P_{ab} \leq 1$ for all $a, b = 1, \ldots, Q$, $P\mathbf{1} \leq \mathbf{1}$ and $\mathbf{1}^TP\leq \mathbf{1}$ element-wise. The network is undirected, so $A_{ji} = A_{ij}$, and $A_{ii} = 0$ (no self-loops). The problem of community detection is then to infer the node labels $\V{c}$ from $A$. Thus we are not really interested in estimation or inference on parameters $\V{\pi}$ and $P$, but, rather we are interested in estimating $\V{c}$. But, it does not mean the two problems are mutually exclusive. In reality, the inferential problem and the community detection problem are quite interlinked. 

The theoretical results of community detection for stochastic block models can be divided into 3 different regimes - 
\begin{itemize}
\item[(a)] $\frac{\E(\mbox{degree})}{\log n} \rar \infty$, equivalent to, $\mathbb{P}[\mbox{there exists an isolated point}] \rar 0$. 
\item[(b)] $\E(\mbox{degree}) \rar \infty$, which means existence of giant component, but also presence of isolated small components from Theorem \ref{thm_giant_comp}. 
\item[(c)] If $\E(\mbox{degree}) = O(1)$, phase boundaries exist, below which community identification is not possible. 
\end{itemize}
Note:
\begin{itemize}
\item[(a)] All of the above mentioned algorithms perform satisfactorily on regime (a).
\item[(b)] None of the above algorithms have been shown to have near perfect probability of success under either regime (b) or (c), for the full parameter space. Some algorithms like \cite{coja2009finding} \cite{bickel2009nonparametric} \cite{chaudhuri2012spectral} \cite{chen2012fitting} are shown to partially work in the sparse setting. Some very recent algorithms include \cite{krzakala2013spectral} \cite{newman2013spectral}.
\end{itemize}

In this paper, we shall only concentrate on stochastic block models. In the future, we shall try to extend our method and results for more general models.

\subsection{Contributions and Outline of the Chapter}
In real life networks, most of the time we seem to see moderately sparse networks \cite{leskovec2007graph} \cite{leskovec2008statistical} \cite{leskovec2009community}. Most of the large or small complex networks we see seem to fall in the (b) regime of Section \ref{sec_th_goal} we describe before, that is, $\E(\mbox{degree}) \rar \infty$. 
We propose a simple algorithm, which performs well in practice in both regimes (b) and (c) and has some theoretical backing 
If degree distribution can identify block parameters then classification using our method should give reasonable result in practice.

Our algorithm is based on graph distance between vertices of the graph. We perform spectral clustering based on the  graph distance matrix of the graph. By looking at the graph distance matrix instead of adjacency matrix for spectral clustering  increases the performance of the community detection, as the normalized distance between cluster centers increases when we go from the adjacency matrix to the graph distance matrix. This helps in community detection even for sparse matrices. We only show theoretical results for stochastic block models. The theoretical proofs are quite intricate and involve careful coupling of the stochastic block model with multi-type branching process to find asymptotic distribution of the typical graph distances. Then, a careful analysis of the eigenvector of the asymptotic graph distance matrix reveals the existence of separation needed for spectral clustering to succeed. This method of analysis has been used for spectral clustering analysis using the adjacency matrix also \cite{sussman2012consistent}, but the analysis is simpler.
 
The rest of the paper is organized as follows. We give a summary of the preliminary results needed in Section \ref{sec_prel}. We present the algorithms in Section \ref{sec_algo}. We give an outline of proof of theoretical guarantee of performance of the method and then the details in Section \ref{sec_geo_theory}. The numerical performance of the methods is demonstrated on a range of simulated networks and on some real world networks in Section \ref{sec_geo_appl}. Section \ref{sec_geo_con} concludes with discussion, and the Appendix contains some additional technical results.

\section{Preliminaries}
\label{sec_prel}
Let us suppose that we have a random graph $G_n$ as the data. Let $V(G_n) = \{v_i, \ldots, v_n\}$ denote the vertices of $G_n$ and $E(G_n) = \{e_1, \ldots, e_m\}$ denote the edges of $G_n$. So, the number of vertices in $G_n$ is $|V(G_n)| = n$ and number of edges of $G_n$ is $|E(G_n)| = m$. Let the adjacency matrix of $G_n$ be denoted by $A_{n\times n}$. For the sake of notational simplicity, from here onwards we shall denote $G_n$ by $G$ having $n$ vertices unless specifically mentioned. We consider the $n$ vertices of $G$ are clustered into $Q$ different communities with each community having size $n_a$, $a=1,\dots, Q$ and $\sum_a n_a = n$. In this paper, we are interested in the problem of \textit{vertex community identification} or \textit{graph partitioning}. That means that we are interested in finding which of the $Q$ different community each vertex of $G$ belongs to. However, the problem is an \textit{unsupervised learning} problem. So, we assume that the data is coming from an underlying model and we try to verify how good `our' \textit{community detection} method works for that model.

\subsection{Model for Community Detection}
As a model for community detection, we consider the stochastic block model. We shall define the stochastic block model shortly, but, we first we shall introduce some more general models, of which stochastic block model is a special case. 

\subsubsection{Bickel-Chen Model}
The general non-parametric model, as described in Bickel, Chen and Levina (2011) \cite{bickel2011method}, that generates the random data network $G$ can be defined by the following equation -
\begin{equation}
\label{eq:nonpar_model}
\Pr(A_{ij} = 1 | \xi_i = u, \xi_j = v) = h_n(u, v) = \rho_n w(u, v) \mathbf{1}(w \leq \rho_n^{-1}),
\end{equation}
where, $w(u, v) \geq 0$, symmetric, $0\leq u, v\leq 1$, $\rho_n \rar 0$. For block models, the latent variable for each vertex $(\xi_1, \ldots, \xi_n)$ can be considered to be coming from a discrete and finite set. Then, each element of that set can be considered to be inducing a partition in the vertex set $V(G_n)$. Thus, we get a model for vertex partitioning, where, the set of vertices can be partitioned into finite number of disjoint classes, but however the partition to which each vertex belongs to is the latent variable in the model and thus unknown. The main goal becomes estimating this latent variable. 

\subsubsection{Inhomogeneous Random Graph Model}
\label{sec_irgm}
The inhomogeneous random graph model (IRGM) was introduced in Bollob\'{a}s et. al. (2007) \cite{bollobas2007phase}. Let $\s$ be a separable metric space equipped with a Borel probability measure $\mu$. For most cases $\s = (0, 1]$ with $\mu$ Lebesgue measure, that means a $U(0,1)$ distribution. The ``kernel" $\kappa$ will be a symmetric non-negative function on $\s \times \s$. For each $n$ we have a deterministic or random sequence $\V{x} = (x_1, \ldots, x_n)$ of points in $\s$. Writing $\gd_x$ for the measure consisting of a point mass of weight 1 at $x$, and
\begin{align*}
\nu_n \equiv \frac{1}{n}\sum_{i=1}^n \gd_{x_i}
\end{align*}
for the empirical distribution of $\V{x}$, it is assumed that $\nu_n$ converges in probability to $\mu$ as $n \rar \infty$, with convergence in the usual space of probability measures on $\s$. One example where the convergence holds is the random case, where the $x_i$ are independent and identically distributed on $\s$ with distribution $\mu$ convergence in probability holds by the law of large numbers. Of course, we do not need $(x_n)_{n\geq1}$ to be defined for every $n$, but only for an infinite set of integers $n$. From here onwards, we shall only focus on this special case, where, $(x_1, \ldots, x_n) \stackrel{iid}{\sim} \mu$.
\begin{definition}
A \textbf{kernel} $\kappa_n$ on a ground space $(\s , \mu)$ is a symmetric non-negative (Borel) measurable function on $\s \times \s$. $\gkp$ is also continuous a.e. on $\s\times \s$. By a \textit{kernel} on a vertex space $(\s, \mu, (x_n)_{n\geq 1})$ we mean a kernel on $(\s, \mu)$.
\end{definition}

Given the (random) sequence $(x_1, \ldots , x_n)$, we let $G (n, \kappa )$ be the random graph $G (n, (p_{ij} ))$ with
\begin{align}
\label{eq_inh_rg}
p_{ij} \equiv \min\{\kappa(x_i,x_j)/n,1\}. 
\end{align}
In other words, $G^{\cV}(n,\kappa)$ has $n$ vertices $\{1,\ldots,n\}$ and, given $x_1,\ldots,x_n$, an edge $ij$ (with
$i \neq j$) exists with probability $p_{ij}$, independently of all other (unordered) pairs $ij$. Based on the graph kernel we can also define an integral operator $T_\kappa$ in the following way
\begin{definition}
\label{def_operator}
The \textbf{integral operator} $T_\kappa:L^2(\s) \rar L^2(\s)$ corresponding to $G(n,\kappa)$, is defined as 
\begin{align*}
T_\kappa f(x)(\cdot) = \int_0^1 \kappa(x,y)f(y)d\mu(y),
\end{align*}
where, $x\in \s$ and any measurable function $f\in L^1(\s)$.
\end{definition}

The random graph $G(n,\kappa)$ depends not only on $\kappa$ but also on the choice of $x_1 , \ldots , x_n$. The freedom of choice of $x_i$ in this model gives some more flexibility than Bickel-Chen model. The asymptotic behavior of $G(n, \kappa)$ depend very much on $\s$ and $\mu$. Many of these key results such as existence of giant component, typical distance, phase transition properties are proved in \cite{bollobas2007phase}. We shall use these results on inhomogeneous random graphs in order to prove results on graph distance for stochastic block models.

Here is further comparison of the Inhomogeneous random graph model (IRGM) with the Bickel-Chen model (BCM), to understand their similarities and dissimilarities -
\begin{itemize}
\item[(a)] In BCM, $(\xi_1, \ldots, \xi_n) \stackrel{iid}{\sim} U(0,1)$ are the latent variables associated with the vertices $(v_1, \ldots, v_n)$ of random graph $G_n$. Similarly, in IRGM, $(x_1, \ldots, x_n) \sim \mu$ are the latent variables associated with the vertices $(v_1, \ldots, v_n)$ of random graph $G_n$. Now, if in IRGM, $(x_1, \ldots, x_n) \stackrel{iid}{\sim} \mu$ then the latent variable structure of the two models become equivalent. 
\item[(b)] In BCM, the conditional probability of connection between two vertices given the value of their latent variables is controlled by the kernel function $h_n(u,v)$. In IRGM, the conditional probability of connection between two vertices given the value of their latent variables is controlled by the kernel function $\frac{\kappa(u,v)}{n}$. 

\item[(c)] So, if $h_n(u,v) = \kappa(u,v)/n$, $\s[(0,1)$ and the underlying measure spaces are same and the measure $\mu$ is a uniform measure on interval $\s = (0, 1)$, then, BCM and IRGM generates graphs from the same distribution. 
In fact, as noted in \cite{bickel2009nonparametric}, if $\s = \R$ and $\mu$ has a positive density with respect to Lebesgue measure, then the (limiting) IRGM is equivalent to Bickel-Chen model with suitable $h_n$.
\item[(d)] For IRGM, let us define 
\begin{align*}
\gl \equiv ||T_\gkp|| \equiv \sup_{f\in L^2(\s), ||f||_{L^2(\s)} = 1} \int_\s\int_\s \kappa(u,v)f(u)f(v) d\mu(u)d\mu(v),
\end{align*}
where, $T_\gkp$ is the operator define in Definition \ref{def_operator} and $||\cdot||$ is the operator norm. In BCM, 
\begin{align*}
\rho_n \equiv \int_0^1\int_0^1 h_n(u,v) dudv.
\end{align*} 
If BCM and IRGM have same underlying measure spaces $(\s=(0,1),\mu=U(0,1))$ and $h_n(u,v) = \kappa(u,v)/n$ and 
\begin{itemize}
\item[Case 1:] $\mathbf{1}$ is the principal eigenfunction of $T_\gkp$, then
\begin{align*}
n\rho_n \rar \gl
\end{align*}
where, $\gl$ is as defined above.
\item[Case 2:] $\mathbf{1}$ is not the principal eigenfunction of $T_\gkp$, then
\begin{align*}
n\rho_n \leq \gl
\end{align*}
\end{itemize}
\end{itemize}

In case of BCM $n\rho_n$ is the natural scaling parameter for the random graph, since, $\E[\mbox{Number of Edges in } G_n] = \frac{1}{2}n\rho_n$. In case of IRGM, $\gl$ is fixed. However, we shall see that the limiting behavior of the graph distance between two vertices of the network becomes dependent on the parameter $\gl$. So, the parameter $\gl$ still remains of importance. We shall henceforth focus on IRGM, with parameter of importance being $\gl$

%
%

\subsubsection{Stochastic Block Model}
The stochastic block model is perhaps the most commonly used and best studied model for community detection. We continue with IRGM framework, so the graph is sparse.
\begin{definition}
\label{def_sbm}
A graph $G^Q(, (P,\V{\pi}))$ generated from \textbf{stochastic block model} (SBM) with $Q$ blocks and parameters $P\in(0,1)^{Q\times Q}$ and $\V{\pi}\in (0,1)^Q$ can be defined in following way - each vertex of graph $G_n$ from an SBM belongs to a community $\V{c} = (c_1, \ldots ,c_n) \in \{1, \ldots,Q\}$ which are drawn independently from the multinomial distribution with parameter $\V{\pi} = (\pi_1, \ldots, \pi_Q)$, where $\pi_i > 0$ for all $i$. Conditional on the labels, the edge variables $A_{ij}$ for $i < j$ are independent Bernoulli variables with
\begin{align}
\label{eq_sbm_adj}
\E[A_{ij}|\V{c}] = P_{c_ic_j} = \min\{\frac{K_{c_ic_j}}{n}, 1\} ,
\end{align}
where $P = [P_{ab}]$ and $K = [K_{ab}]$ are $Q \times Q$ symmetric matrices. $P$ is known as the \textbf{connection probability} matrix and $K$ as the \textbf{kernel} matrix for the connection. So, we have $P_{ab} \leq 1$ for all $a, b = 1, \ldots, Q$, $P\mathbf{1} \leq \mathbf{1}$ and $\mathbf{1}^TP\leq \mathbf{1}$ element-wise.
\end{definition}
The network is undirected, so $A_{ji} = A_{ij}$, and $A_{ii} = 0$ (no self-loops). The problem of community detection is then to infer the node labels $\V{c}$ from $A$. Thus we are not really interested in estimation or inference on parameters $\V{\pi}$ and $P$, but, rather we are interested in estimating $\V{c}$. But, it does not mean the two problems are mutually exclusive, in reality, the inferential problem and the community detection problem are quite interlinked. 

We can see that SBM is a special case of both Bickel-Chen model and IRGM. In IRGM, if we consider $\s$ to be a finite set, $(x_1, \ldots, x_n)\in [Q]^n$ ($[Q] = \{1, \ldots, Q\}$) with $x_i \stackrel{iid}{\sim} Mult(n, \V{\pi})$ and kernel $\kappa:[Q]\rar[Q]$ as $\kappa(a,b) = K_{ab}$ ($a, b = 1, \ldots, Q$), then the resulting IRGM graph follows stochastic block model. So, for SBM we can define an \textit{integral operator} on $[Q]$ with measure $\{\pi_1, \ldots, \pi_Q\}$. 
\begin{definition}
\label{def_op_sbm}
The \textbf{integral operator} $T_K:\ell^1(\s) \rar \ell^1(\s)$ corresponding to \\ 
$G^Q(n, (P,\V{\pi}))$, is defined as 
\begin{align*}
(T_K(x))_a = \sum_{b=1}^Q K_{ab}\pi_bx_b, \mbox{ for } a = 1, \ldots, Q
\end{align*}
where, $x\in \R^Q$.
\end{definition}

The stochastic block model has deep connections with Multi-type branching process, just as, Erod\"{o}s-R\'{e}nyi random graph model (ERRGM) has connections with the branching process. Let us introduce branching process first.

\subsection{Multi-type Branching Process}
\label{sec_br_proc}
We shall try to link network formed by SBM with the tree network generated by multi-type Galton-Watson branching process. In our case, the Multi-type branching process (MTBP) has type space $S = \{1, \ldots, Q\}$, where a particle of type $a \in S$ is replaced in the next generation by a set of particles distributed as a Poisson process on $S$ with intensity $(K_{ab}\pi_b)_{b = 1}^Q$. We denote this branching process, started with a single particle of type $a$, by $\cb_{K,\pi}(a)$. We write $\cb_{K,\pi}$ for the same process with the type of the initial particle random, distributed according to $\V{\pi}$.

\begin{definition}
\begin{itemize}
\item[(a)] Define $\rho_k (K,\pi; a)$ as the probability that the branching process \\ 
$\cb_{K,\pi}(a)$ has a total population of exactly $k$ particles.
\item[(b)] Define $\rho_{\geq k}(K,\pi; a)$ as the probability that the total population is at least $k$. 
\item[(c)] Define $\rho(K,\pi; a)$ as the probability that the branching process survives for eternity. 
\item[(d)] Define,
\begin{align}
\label{eq_br_surv}
\rho_k(K,\pi) \equiv \sum_{a=1}^Q\rho_k(K,\pi; a)\pi_a, \ \ \ \rho\equiv \rho(K,\pi) \equiv \sum_{a=1}^Q \rho(K,\pi; a)\pi_a
\end{align}
and define $\rho_{\geq k}(K)$ analogously. Thus, $\rho(K,\pi)$ is the \textbf{survival probability} of the branching process $\cb_{K,\pi}$ given that its initial distribution is $\V{\pi}$
\end{itemize}
\end{definition}
If the probability that a particle has infinitely many children is 0, then $\rho(K,\pi; a)$ is equal to $\rho_{\infty}(a)$, the probability that the total population is infinite.
As we shall see later, the branching process $\cb_{K,\pi}(a)$ arises naturally when exploring a component of $G_n$ starting at a vertex of type $a$; this is directly analogous to the use of the single-type Poisson branching process in the analysis of the Erd\"{o}s-R\'{e}nyi graph $G(n, c/n)$. 

\subsection{Known Results for Stochastic Block Model}
\label{sec_known_sbm}
The performance of community detection algorithms depends on the parameters $\V{\pi}$ and $P$. We refer to Definition \ref{def_sbm} for definition of stochastic block models. An important condition that we usually put on parameter $P$ is \textit{irreducibility}. 
\begin{definition}
\label{def_irreducibility}
A connection matrix $P$ on a $\s = \{1, \ldots, Q\}$ is \textbf{reducible} if there exists $A \subset \s$ with $0<|A|<Q$ such that $P = 0$ a.e. on $A \times (\s - A)$; otherwise $P$ is \textbf{irreducible}. Thus $P$ is \textbf{irreducible} if $A \subseteq \s$ and $P = 0$ a.e. on $A \times (\s -A)$ implies $|A| = 0$ or $|A| = Q$.
\end{definition}

So, the results on existence of giant components in \cite{bollobas2007phase} also apply for SBM. The following theorem describes the result on existence of giant components.
\begin{theorem}[\cite{bollobas2007phase}]
 \label{thm_giant_comp}
Let us define operator $T_K$ as in definition \ref{def_op_sbm},
 \begin{itemize}
\item[(i)] If $||T_K|| \leq 1$ ($||\cdot||$ refer to operator norm), then the size of largest component is $o_P(n)$, while if $||T_K|| > 1$, then the size of largest component is $\Theta_P(n)$ whp.
\item[(ii)] If $P$ is irreducible, then $\frac{1}{n}(\mbox{Size of largest component}) \rar \V{\pi}^T\rho$, where, $\rho\in[0,1]^Q$ is the survival probability as defined in \eqref{eq_br_surv}.
 \end{itemize}
\end{theorem} 

The theoretical results on community detection depend on the 3 different regime on which the generative model is based on - 
\begin{itemize}
\item[(a)] $\frac{\E(\mbox{degree})}{\log n} \rar \infty$, equivalent to, $\mathbb{P}[\mbox{there exists an isolated point}] \rar 0$. In this setting, there are several algorithms, such as those described in Section 1, can identify correct community with high probability under quite relaxed conditions on parameters $P$ and $\V{\pi}$. See \cite{chaudhuri2012spectral} (Theorem 2 and 3), \cite{rohe2011spectral} (Theorem 3.1), \cite{coja2009finding} (Theorem 1).
\item[(b)] $\E(\mbox{degree}) \rar \infty$, which means existence of giant component, but also presence of isolated small components from Theorem \ref{thm_giant_comp}. In this setting, algorithms proposed in \cite{coja2009finding}, \cite{chen2012fitting} is proved to identify community labels that are highly correlated with original community labels with high probability.
\item[(c)] If $\E(\mbox{degree}) = O(1)$, phase boundaries exist, below which community identification is not possible. These results and rigorous proof are given in \cite{mossel2012stochastic}. The results can be summarized for 2-block model with parameters $P_{11} = a, P_{12}=b, P_{22}=a$ as
\begin{theorem}[\cite{mossel2012stochastic}]
\label{thm_phase_bound}
\begin{itemize}
\item[(i)] If $(a-b)^2 < 2(a+b)$ then probability model of SBM and ERRGM with $p = \frac{a+b}{2n}$ are mutually contiguous. Moreover, if $(a-b)^2<2(a+b)$, there exists no consistent estimators of $a$ and $b$.
\item[(ii)] If $(a-b)^2 > 2(a+b)$ then probability model of SBM and ERRGM with $p = \frac{a+b}{2n}$ are asymptotically orthogonal.
\end{itemize}
\end{theorem}
So, in the range $(a-b)^2>2(a+b)$, there should exists an algorithm which identifies highly correct clustering with high probability at least within the giant components.
\end{itemize}

\section{Algorithm}
\label{sec_algo}
The algorithm we propose depend on the graph distance or geodesic distance between vertices in a graph.
\begin{definition}
\textbf{Graph distance} or \textbf{Geodesic distance} between two vertices $i$ and $j$ of graph $G$ is given by the length of the shortest path between the vertices $i$ and $j$, if they are connected. Otherwise, the distance is infinite.
\end{definition}
So, for any two vertices $u, v\in V(G)$, graph distance, $d_g$ is defined by
\begin{eqnarray*}
\label{eq_geo_dis}
d_g(u, v) &=&\left\{
\begin{array}{ll}
    |V(e)|, & \mbox{  if  } e \mbox{ is the shortest path connecting } u \mbox{ and } v\\
    \infty, & u \mbox{ and } v \mbox{ are not connected}
     \end{array}
     \right.
\end{eqnarray*}
For sake of numerical convenience, we shall replace $\infty$ by a large number for value of $d_g(u, v)$, when, $u$ and $v$ are not connected. The main steps of the algorithm can be described as follows
\begin{enumerate}
\item[1.] Find the graph distance matrix $D = [d_g(v_i, v_j)]_{i, j=1}^n$ for a given network but with distance upper bounded by $k\log n$. Assign non-connected vertices an arbitrary high value $B$.
\item[2.] Perform hierarchical clustering to identify the giant component $G^C$ of graph $G$. Let $n_C = |V(G^C)|$.
\item[3.] Normalize the graph distance matrix on $G^C$, $D^C$ by 
\begin{align*}
\bar{D}^C = -\left(I-\frac{1}{n_C}\mathbf{1}\mathbf{1}^T\right)(D^C)^2\left(I-\frac{1}{n_C}\mathbf{1}\mathbf{1}^T\right) 
\end{align*}
\item[4.] Perform eigenvalue decomposition on $\bar{D}^C$.
\item[5.] Consider the top $Q$ eigenvectors of normalized distance matrix $\bar{D}^C$ and $\tilde{\bw}$ be the $n\times Q$ matrix formed by arranging the $Q$ eigenvectors as columns in $\tilde{\bw}$. Perform $Q$-means clustering on the rows $\tilde{\bw}$, that means, find an $n\times Q$ matrix $\bc$, which has $Q$ distinct rows and minimizes $||\bc - \tilde{\bw}||_F$.
\item[6.] (Alternative to 5.) Perform Gaussian mixture model based clustering on the rows of $\tilde{\bw}$, when there is an indication of highly-varying average degree between the communities.
\item[7.] Let $\hat{\xi} : V \mapsto [Q]$ be the block assignment function according to the clustering of the rows of $\tilde{\bw}$ performed in either Step 5 or 6.
\end{enumerate}
Here are some important observations about the algorithm -
\begin{itemize}
\item[(a)] There are standard algorithms for graph distance finding in the algorithmic graph theory literature. In algorithmic graph theory literature the problem is known as the \textbf{all pairs shortest path} problem. The two most popular algorithms are Floyd-Warshall \cite{floyd1962algorithm} \cite{warshall1962theorem} and Johnson's algorithm \cite{johnson1977efficient}. The time complexity of the Floyd-Warshall algorithm is $O(n^3)$, where as, the time complexity of Johnson's algorithm is $O(n^2\log n + ne)$ \cite{leiserson2001introduction} ($n=|V(G_n)|$ and $e=|E(G_n)|$). So, for sparse graphs, Johnson's algorithm is faster than Floyd-Warshall. Memory storage is also another issue for this algorithm, since the algorithm involves a matrix multiplication step of complexity $\Omega(n^2)$. Recently, there also has been some progress on parallel implementation of all-pairs shortest path problem \cite{solomonik2012minimizing} \cite{bulucc2010solving} \cite{habbal1994decomposition}, which addresses both memory and computation aspects of the algorithm and lets us scale the algorithm for large graphs, both dense and sparse.
\item[(b)] The Step 3 of the algorithm is nothing but the classical multi-dimensional scaling (MDS) of the graph distance matrix. In MDS, we try to find vectors $(x_1, \ldots, x_n)$, where, $x_i\in \R^Q$, such that, 
\begin{align*}
\sum_{i,j=1}^n \left(||x_i - x_j||_2 - (D^C)_{ij}\right)^2
\end{align*}
is minimized. The minimizer is attained by the rows of the matrix formed by the top $Q$ eigenvectors of $\bar{D}^C$ as columns. So, performing spectral clustering on $\bar{D}^C$ is the same as performing $Q$-means clustering on the multi-dimensional scaled space.

Instead of $\bar{D}^C$, we could also use the matrix $(D^C)^2$, but then, the topmost eigenvector does not carry any information about the clustering. Similarly, we can also use the matrix $D^C$ directly for spectral clustering, but, in that case, $D^C$ is not a positive semi-definite matrix and as a result we have to consider the eigenvectors corresponding to largest absolute eigenvalues (since eigenvalues can be negative). 
\item[(c)] In the Step 5 of the algorithm $Q$-means clustering if the expected degree of the blocks are equal. However, if the expected degree of the blocks are different, it leads to multi scale behavior in the eigenvectors of the normalized distance matrix. So, we perform Gaussian Mixture Model (GMM) based clustering instead of $Q$-means to take into account the multi scale behavior.
\end{itemize}


\section{Theory}
\label{sec_geo_theory}
Let us consider that we have a random graph $G_n$ as the data. Let $V(G_n) = \{v_i, \ldots, v_n\}$ denote the vertices of $G_n$ and $E(G_n) = \{e_1, \ldots, e_m\}$ denote the edges of $G_n$. So, the number of vertices in $G_n$ is $|V(G_n)| = n$ and number of edges of $G_n$ is $|E(G_n)| = m$. Let the adjacency matrix of $G_n$ be denoted by $A_{n\times n}$. For sake of notational simplicity, from here onwards we shall denote $G_n$ by $G$ having $n$ vertices unless specifically mentioned. There are $Q$ communities for the vertices and each community has $(n_a)_{a=1}^Q$ number of vertices. In this paper, we are interested in the problem of \textit{vertex community identification} or \textit{graph partitioning}. However, the problem is an \textit{unsupervised learning} problem. So, we assume that the data is coming from an underlying model and we try to verify how good `our' \textit{community detection} method works for that model.

The theoretical analysis of the algorithm has two main parts -
\begin{itemize}
\item[I.] Finding the limiting distribution of graph distance between two typical vertices of type $a$ and type $b$ (where, $a, b = 1, \ldots, Q$). This part of the analysis is highly dependent on results from multi-type branching processes and their relation with stochastic block models. The proof techniques and results are borrowed from \cite{bollobas2007phase}, \cite{bhamidi2011first} and \cite{athreya1972branching}.
\item[II.] Finding the behavior of top $Q$ eigenvectors of the graph distance matrix $D$ using the limiting distribution of the typical graph distances. This part of analysis is highly dependent on perturbation theory of linear operators. The proof techniques and results are borrowed from \cite{kato1995perturbation}, \cite{chatelin1983spectral} and \cite{sussman2012consistent}.
\end{itemize}
\subsection{Results of Part I}
\label{sec_theo_res}
We shall give limiting results for \textit{typical distance} between vertices in $G_n$. If $u$ and $v \in V(G_n)$ are two vertices in $G_n$, which has been selected uniformly at random from type $a$ and type $b$ respectively, where, $a, b = 1, \ldots, Q$ are the different communities. Then, the graph distance between $u$ and $v$ is $d_G(u, v)$. Now, the operator that controls the process is $T_K$ as defined in Definition \ref{def_op_sbm}. $T_K$ is another representation of the matrix $\tilde{K}_{Q\times Q}$, which is defined as
\begin{align}
\label{eq_op_mat}
\til{K}_{ab} \equiv \pi_aK_{ab}\pi_b, \mbox{ for } a, b = 1, \ldots, Q
\end{align}
The matrix $\tilde{K}$ defines the quadratic form for $T_K:\ell^1(\s,\pi)\rar\ell^1(\s,\pi)$. So, we have that 
\begin{align}
\label{eq_op_eig}
\gl \equiv ||T_K|| = \gl_{max}(\til{K}).
\end{align}
The relation between $\gl$ and $\E[\mbox{number of Edges in } G_n]$ is given Section \ref{sec_irgm}. Here, we use $\gl$ as the scaling operator, not either average, minimum or maximum degree of vertices as used in \cite{sussman2012consistent} and \cite{rohe2011spectral}. But, we already know that, if the graph is \textit{homogeneous}, then, $\E[\mbox{number of Edges in } G_n] = \frac{1}{2}\gl$ and otherwise $\E[\mbox{number of Edges in } G_n] \leq \gl$.

Let us also denote, $\nu\in \R^Q$ as the eigenvector of $\til{K}$ corresponding to $\gl$. We at first, try to find an asymptotic bound on the graph distance $d_G(u,v)$ for vertices $u,v\in V(G)$.
\begin{theorem}
\label{thm_geo_dis_bnd}
Let $\gl > 1$ (defined in Eq. \eqref{eq_op_eig}), then, the graph distance $d_G(u, v)$ between two uniformly chosen vertices of type $a$ and $b$ respectively, conditioned on being connected, satisfies the following asymptotic relation -
\begin{itemize}
\item[(i)] If $a = b$ 
\begin{align}
\label{eq_geo_diff1a}
\Pr\left[d_G(u,v) \leq (1- \gve)\frac{\log n}{\log (\pi_aK_{aa})}\right] = o(1)
\end{align}
\begin{align}
\label{eq_geo_diff1b}
\Pr\left[d_G(u,v) \geq (1+ \gve)\frac{\log n\pi_a}{\log (\pi_aK_{aa})}\right] = o(1)
\end{align}
\item[(ii)] 
If $a\neq b$,
\begin{align}
\label{eq_geo_diff2a}
\Pr\left[d_G(u,v) \leq (1- \gve)\frac{\log n}{\log |\gl|}\right] = o(1)
\end{align}
\begin{align}
\label{eq_geo_diff2b}
\Pr\left[d_G(u,v) \geq (1+ \gve)\frac{\log n}{\log |\gl|}\right] = o(1)
\end{align}
\end{itemize}
\end{theorem}
Now, let us consider the limiting operator $\D$ defined as 
\begin{definition}
\label{def_op_norm}
The \textbf{normalized limiting matrix} is an $n\times n$ matrix,  $\D$, which in limit as $n\rar \infty$ becomes an operator on $l_2$ (space of convergent sequences), is defined as $\D = [\D_{ij}]_{i,j = 1}^n$, where, 
\begin{eqnarray*}
\label{eq_geo_lim1}
\D_{ij} &=&\left\{
\begin{array}{ll}
    \frac{1}{\log |\gl|}, & \mbox{  if  type of } v_i = a\neq b = \mbox{ type of } v_j\\
    \frac{1}{\log (\pi_aK_{aa})}, & \mbox{  if  type of } v_i = \mbox{ type of } v_j = a
     \end{array}
     \right.
\end{eqnarray*}
and $\D_{ii} = 0$ for all $i = 1, \ldots, n$.\\
The \textbf{graph distance matrix} $\bld$ can be defined as
\begin{align*}
\bld = [d(v_i, v_j)]_{i,j = 1}^n.
\end{align*}
\end{definition}
In Theorem \ref{thm_geo_dis_bnd} we had a point-wise result, so, we combine these point-wise results to give a matrix result -
\begin{theorem}
\label{thm_geo_dis}
Let $\gl = ||T_K|| > 1$, then, within the big connected component,
\begin{align*}
\Pr\left[\left|\left|\frac{\bld}{\log n} - \D\right|\right|_F \leq O( n^{1-\gve})\right] = 1-o(1)
\end{align*}
\end{theorem}
Thus, the above theorem gives us an idea about the limiting behavior of the normalized version of geodesic matrix $\bld$. 
%

\subsubsection{Sketch of Analysis of Part I}
\label{sec_sum_geo}
A rough idea of the proof of part I is as follows. Fix two vertices, say 1 and 2, in the giant component. Think of a branching process starting from vertices of type 1 and 2, so that at time $t$, $ \cb_{P\pi}(a) (t)$ is the branching process tree from vertex of type $a$ and includes the shortest paths to all vertices in the tree at or before time $t$ from vertex $a$, $a = 1, 2$. When these two trees meet via the formation of an edge $(v_1 , v_2 )$ between two vertices $v_1 \in \cb_{P\pi}(1)(\cdot)$ and $v_2 \in \cb_{P\pi}(2)(\cdot)$, then the shortest-length path between the two vertices 1 and 2 has been found. If $D_n (v_a )$, $a = 1, 2$, denotes the number of edges between the source $a$ and the vertex
$v_a$ along the tree $\cb_{P\pi} (a)$, then the graph distance $d_n (1, 2)$ is given by
\begin{align}
\label{eq_gr_dist}
d_n(1, 2) = D_n(v_1) + D_n(v_2) + 1
\end{align}
The above idea is indeed a very rough sketch of our proof and it follows from the graph distance finding idea developed in \cite{bollobas2007phase}. In the paper, we embed the SBM in a multi-type branching process (MTMBP) or a single-type marked branching process (MBP), depending on whether the types of two vertices are same or not. The offspring distribution is binomial with parameters $n - 1$ and kernel $P$ (see Section \ref{sec_det_geo}). With high probability, the vertex exploration process in the SBM can be coupled with two multi-type branching processes, bounding the vertex exploration process on SBM on both sides. Now, using the property of the two multi-type branching processes, we can bound the number of vertices explored in the vertex exploration process of a SBM graph and infer about the asymptotic limit of the graph distance.

With the above sketch of proof can be organized as follows.
\begin{enumerate}
\item We analyze various properties of a Galton-watson process conditioned on non-extinction, including times to grow to a particular size. In this branching process, the offspring will have a Poisson distribution.
\item We introduce multi-type branching process trees with binomially distributed offspring and make the connection between these trees and the SBM. We bound the vertices explored for an SBM graph, starting from a fixed vertex, by considering a muti-type branching process coupled with it.
\item We bound the geodesic distance using the number of vertices explored in the coupled multi-type branching processes within a certain generation. The limiting behavior of the generation give us the limiting behavior of graph distance.
\item The whole analysis is true for IRGM. So, the results are true for SBM with increasing block numbers and degree-corrected block models also.
\end{enumerate}
The idea of the argument is quite simple, but making these ideas rigorous takes some technical work, particularly because we need to condition on our vertices being in the giant component.

\subsection{Results of Part II} 
So, from Part I of the analysis, we get an idea about the point-wise asymptotic convergence of the matrix $\bld = [d(v_i, v_j)]_{i,j = 1}^n$ to the normalized limiting operator $\D$, defined in Definition \ref{def_op_norm}.

The limiting matrix $\D$ can also be written in terms of limiting low-dimensional matrix, $\cd$, which is defined as follows -
\begin{definition}
The \textit{limiting kernel matrix}, $\cd_{Q\times Q}$ is defined as
\begin{eqnarray*}
\label{eq_geo_lim2}
\cd_{ab} &=&\left\{
\begin{array}{ll}
    \frac{1}{\log |\gl|}, & \mbox{  if  } a\neq b\\
    \frac{1}{\log (\pi_aK_{aa})}, & \mbox{  if  } a=b
     \end{array}
     \right.
\end{eqnarray*}
\end{definition}
So, we can see that if $\bj_{n\times n} = \mathbf{1}\mathbf{1}^T$ is an $n\times n$ matrix of all ones, then, there exists a permutation of rows of $\D$, which is obtained by multiplying $\D$ with permutation matrix $\br$, such that,
\begin{align}
\label{eq_lim_rel}
\D\br = \cd\star\bj - Diag(\tilde{d}) \equiv [\cd_{ab}\bj_{ab}]_{a,b=1}^Q - Diag(\tilde{d})
\end{align}
where, $[\bj_{ab}]_{a,b=1}^Q$ is a $Q\times Q$ partition of $\bj$ in the following way - the rows and columns are partitioned in similar fashion according to $(n_1, \ldots, n_Q)$. Note that, $(n_a)_{a=1}^Q$ are the number of vertices of type $a$ in the graph $G_n$. So, $\bj_{ab}$ is an $n_a\times n_b$ matrix of all ones. $\tilde{d}$ is a vector of length containing $n_a$ elements of value $\frac{1}{\log (\pi_aK_{aa})}$, $a = 1, \ldots, Q$. Note that product $\star$ can also be seen as a Khatri-Rao product of two partitioned matrices \cite{khatri1968solutions}.

Now, we assume some conditions on the limiting low-dimensional matrix $\cd$.
\begin{itemize}
\item[(C1)] We assume that $\gl < \min_a\{\pi_aK_{aa}\}$, where, $\gl$ defined in Eq. \eqref{eq_op_eig} is the principal eigenvector of operator $T_K$ defined in Def. \ref{def_op_sbm} or the matrix $\tilde{K}$ defined in Eq. \eqref{eq_op_mat}.
\item[(C2)] The eigenvalues of $\cd$, $\gl_1(\cd) \geq \cdots\geq\gl_Q(\cd)$, satisfy the condition that there exists an constant $\ga$, such that, $0 < \ga \leq \gl_Q(\cd)$. 

%

\item[(C3)] The eigenvectors of $\cd$, $(v_1(\cd), \ldots, v_Q(\cd))$ corresponding to $\gl_1, \ldots, \gl_Q$, satisfy the condition that there exists a constant $\gb$, such that, rows of the $Q\times Q$ matrix $\mathbf{V} = [v_1 \cdots v_Q]$, represented as $(u_1, \ldots, u_Q)$ ($u_a\in \R^Q)$, satisfies the condition $0 < \gb \leq ||u_a - u_b||_2$ for all pairs of rows of $\mathbf{V}$.  

\item[(C4)] The number of vertices in each type $(n_1, \ldots, n_Q)$, satisfy the condition that there exists a constant $\gt$ such that $0<\gt<\frac{n_a}{n}$ for all $a=1, \ldots, Q$ and all $n$. 
\end{itemize}

\begin{theorem}
\label{thm_mis}
Under the conditions (C1)-(C4), suppose that the number of blocks $Q$ is known. Let $\hat{\xi} : V \mapsto [Q]$ be the block assignment function according to a clustering of the rows of $\tilde{\bw}^{(n)}$ satisfying algorithm in Section \ref{sec_algo} and $\xi : V \mapsto [Q]$ be the actual assignment. Let $\mathcal{P}_Q$ be the set of permutations on $[Q]$. With high probability and for large $n$ it holds that
\begin{align}
\label{eq_miscl}
\min_{\pi\in\mathcal{P}_Q} |\{u \in V: \xi(u)\neq \pi(\hat{\xi}(u))\}| = O(n^{1/2-\gve})
\end{align}
\end{theorem}

\subsubsection{Sketch of Proof of Part II}
\label{sec_sum_pert}
We can consider the limiting distribution of the graph distance matrix as $\bld$ which was proposed in Theorem \ref{thm_geo_dis}, with $(\bld_{ij}) = d_G(v_i, v_j)$, where, $v_i, v_j \in V(G)$. Our goal is to show that the eigenvectors of $\bld$ or normalized version of it, converge to eigenvectors of $\cd$ or $\D$. For that reason, we use the perturbation theory of operators, as given in Kato \cite{kato1995perturbation} and Davis-Kahan \cite{davis1970rotation}. The steps are as follows
\begin{itemize}
\item We use Davis-Kahan to show convergence of eigenspace $\tilde{\bw}$, formed by top $Q$ eigenvectors of $\bld/\log n$ to $\bw\br$, where, $\bw$ is the eigenspace formed by the top $Q$ eigenvectors of $\D$ and $\br$ is some orthogonal permutation matrix, which permutes the rows of $\bw$.
\item We show by contradiction that if the clustering assignment makes too many mistakes then the rate of convergence of $\tilde{\bw}$ to $\bw\br$ would be violated.
\end{itemize} 

\subsection{Branching Process Results}
\label{sec_prop_br}
The branching process $\cb_K(a)$ is a multi-type Galton-Watson branching processes with type space $\s \equiv \{1, \ldots, Q\}$, a particle of type $a \in \s$ is replaced in the next generation by its ``children", a set of particles whose types are distributed as a Poisson process on $\s$ with intensity $\{K_{ab}\pi_b\}_{b=1}^Q$. Recall the parameters $K\in \R^{Q\times Q}$ and $\pi\in [0, 1]^Q$ with $\sum_{a = 1}^Q\pi_a = 1$, from the definition of Stochastic block model in equation \eqref{eq_sbm_adj}. The zeroth generation of $\cb_K(a)$ consists of a single particle of type $a$. Also, the branching process $\cb_K$ is just the process $\cb_K(a)$ started with a single particle whose (random) type is distributed according to the probability measure $(\pi_1, \ldots, \pi_Q)$.

Let us recall our notation for the survival probabilities of particles in $\cb_K(a)$. We write $\rho_k(K; a)$ for the probability that the total population consists of exactly $k$ particles, and $\rho_{\geq k}(K; a)$ for the probability that the total population contains at least $k$ particles. Furthermore, $\rho(K; a)$ is the probability that the branching process survives for eternity.
We write $\rho_k(K), \rho_{\geq k}(K)$ and $\rho(K)$ for the corresponding probabilities for $\cb_K$ , so that, e.g., $\rho_k(K)= \sum_{a = 1}^Q \rho_k(K; a)\pi_a$.

Now, we try to find a coupling relation between \textit{neighborhood exploration process} of a vertex of type $a$ in stochastic block model and multi-type Galton-Watson process, $\cb(a)$ starting from a vertex of type $a$. 

We assume all vertices of graph $G_n$ generated from a stochastic block model has been assigned a community or type $\xi_i$ (say) for vertex $v_i \in V(G_n)$. By \textit{neighborhood exploration process} of a vertex of type $a$ in stochastic block model, we mean that we start from a random vertex $v_i$ of type $a$ in the random graph $G_n$ generated from stochastic block model. Then, we count the number of vertices of the random graph $G_n$ are neighbors of $v_i$, $N(v_i)$. We repeat the neighborhood exploration process by looking at the neighbors of the vertices in $N(v_i)$. We continue until we have covered all the vertices in $G_n$. Since, we either consider $G_n$ connected or only the giant component of $G_n$, the neighborhood exploration process will end in finite steps but the number of steps may depend on $n$. 
\begin{lemma}
\label{lemma_brpr_sbm}
Within the giant component, the neighborhood exploration process for a stochastic block model graph with parameters $(P, \pi) = (K/n,\pi)$, can be bounded with high probability by two multi-type branching processes with kernels $(1 - 2\gre)K$ and $(1 + \gre)K$ for some $\gre > 0$.
\end{lemma}
\begin{proof}
The proof is given in Appendix A1 and follows from Lemma 9.6 \cite{bollobas2007phase}.
\end{proof}

Now, we restrict ourselves to the giant component only. So, if we condition that the exploration process does not leave the giant component, it is same as conditioning that the branching process does not die out. Under this additional condition, the branching process can be coupled with another branching process with a different kernel. The kernel of that branching process is given in following lemma.
\begin{lemma}
\label{lemma_brpr_cond}
If we condition a branching process, $\cb_{K\pi}$ on survival, the new branching process has kernel $\left(K_{ab}\left(\rho(K;a) + \rho(K;b) - \rho(K; a)\rho(K; b)\right)\right)_{a,b = 1}^Q$. 
\end{lemma}
\begin{proof}
The proof is given in Appendix A2 and follows from Section 10 of \cite{bollobas2007phase}.
\end{proof}

Now, we shall try to prove the limiting behavior of typical distance between vertices $v$ and $w$ of $G_n$, where, $v, w\in V(G_n)$. We first try to find a lower bound for distance between two vertices. We shall separately give an upper bound and lower bound for distance between two vertices of same type and different types. The result on lower bound in proved in Lemma \ref{lm_low_bnd}.

\begin{lemma}
\label{lm_low_bnd}
For vertices $v,w\in V(G)$, if 
\begin{itemize}
\item[(a)] type of $v = a \neq b = $ type of $w$ (say) and $\gl \equiv ||T_K|| > 1$, then, 
\begin{align*}
\E\left|\left\{\{v, w\}: d_G(v, w) \leq (1 - \gve)\log n/\log |\gl|\right\}\right| = O(n^{2-\gve})
\end{align*}
and so
\begin{align*}
\left|\left\{\{v, w\}: d_G(v, w) \leq (1 - \gve)\frac{\log n}{\log \gl}\right\}\right| \leq O(n^{2-\gve/2}) \mbox{ with high probability}
\end{align*}
\item[(b)] type of $v = $ type of $w = a$ (say), $\gl \equiv ||T_K|| < \pi_aK_{aa}$ and $\gl > 1$, then, 
\begin{align*}
\E\left|\left\{\{v, w\}: d_G(v, w) \leq (1 - \gve)\log n/\log (\pi_aK_{aa})\right\}\right| = O(n^{2-\gve})
\end{align*}
and so
\begin{align*}
\left|\left\{\{v, w\}: d_G(v, w) \leq (1 - \gve)\frac{\log n}{\log (\pi_aK_{aa})}\right\}\right| \leq O(n^{2-\gve/2}) \mbox{ with high probability}
\end{align*}
\end{itemize}
\end{lemma}
\begin{proof}
The proof is given in Appendix A3 and follows from Lemma 14.2 of \cite{bollobas2007phase}.
\end{proof}

Now, we first try to upper bound the typical distance between two vertices of the same type. For the same type vertices, we just focus on the subgraph of the original graph from stochastic block model having vertices of same type. So, in Lemma \ref{lm_up_bnd1}, the graph $G_n$ is the subgraph of the original graph containing only the vertices of the same type. So, the coupled branching process on that graph automatically becomes a single-type branching process.
\begin{lemma}
\label{lm_up_bnd1}
For vertices $v,w\in V(G)$, if type of $v = $ type of $w = a$ (say) 
\begin{align*}
\Pr\left(d_G(v, w) < (1 + \gve)\frac{\log (n\pi_a)}{\log (\pi_aK_{aa})}\right) = 1 - exp(-\Omega(n^{2\eta}))
\end{align*}
conditioned on the event that the branching process $\cb_{K_{aa}}(a)$ survives.
\end{lemma}
\begin{proof}
The proof is given in Appendix A4 and follows from Lemma 14.3 of \cite{bollobas2007phase}.
\end{proof}

Now, let us try to upper bound the typical distance between two vertices of different types. So, in Lemma \ref{lm_up_bnd2}, the graph $G_n$ is the original graph containing with vertices of the different types. So, the coupled branching process on that graph becomes a multi-type branching process.

\begin{lemma}
\label{lm_up_bnd2}
Let us have $\gl \equiv ||T_K|| = \gl_{\max}(\tilde{K}) > 1$ from Eq \eqref{eq_op_eig}. For uniformly selected vertices $v,w\in V(G)$,
\begin{align*}
\Pr\left(d_G(v, w) < (1 + \gve)\frac{\log n}{\log \gl}\right) = 1 - exp(-\Omega(n^{2\eta}))
\end{align*}
conditioned on the event that the branching process $\cb_{K}$ survives.
\end{lemma}
\begin{proof}
The proof is given in Appendix A5 follows from Lemma 14.3 of \cite{bollobas2007phase}.
\end{proof}

\subsection{Proof of Theorem \ref{thm_geo_dis_bnd} and Theorem \ref{thm_geo_dis}}
\label{sec_det_geo}
\subsubsection{Proof of Theorem \ref{thm_geo_dis_bnd}}
We shall try to prove the limiting behavior of typical graph distance in the giant component as $n \rar \infty$. The Theorem essentially follows from Lemma \ref{lm_low_bnd} - \ref{lm_up_bnd2}. Under the conditions mentioned in the Theorem, part (a) follows from Lemma \ref{lm_low_bnd}(b) and \ref{lm_up_bnd1} and part (b) follows from Lemma \ref{lm_low_bnd}(a) and \ref{lm_up_bnd2}.
\subsubsection{Proof of Theorem \ref{thm_geo_dis}}
From the definition \ref{def_op_norm}, we have that $\bld_{ij} = $ graph distance between vertices $v_i$ and $v_j$, where, $v_i, v_j \in V(G_n)$.
\begin{itemize}
\item[Case 1:] For the case when type of $v_i = $ type of $v_j = a$ (say)\\
From Lemma \ref{lm_low_bnd}(b), we get for any vertices $v$ and $w$ of same type $a$ with high probability, 
\begin{align*}
\left|\left\{\{v, w\}: d_G(v, w) \leq (1 - \gve)\frac{\log n}{\log (\pi_aK_{aa})}\right\}\right| \leq O(n^{2-\gve}).
\end{align*}
Also from Lemma \ref{lm_up_bnd1}, we get, for any vertices $v$ and $w$ of same type $a$
\begin{align*}
\Pr\left(d_G(v, w) < (1 + \gve)\frac{\log n}{\log (\pi_aK_{aa})}\right) = 1 - exp(-\Omega(n^{2\eta}))
\end{align*}
So, we have that, for $v_i, v_j$ having same type $a$, with high probability,
\begin{align*}
\left(\frac{\bld_{ij}}{\log n} - \D_{ij}\right)^2 \leq \mbox{Constant }\gve^2
\end{align*}
Since, $\gre = O(n^{-1/2})$ by Eq \eqref{eq_gre} and $(1 - exp(-\Omega(n^{2\eta})))^{n^2} \rar 1$ as $n\rar \infty$,
\begin{align*}
\sum_{i,j = 1: type(v_i)=type(v_j)}^n\left(\frac{\bld_{ij}}{\log n} - \D_{ij}\right)^2 \leq \gve^2.O((n\pi_a)^2) = O(n)
\end{align*}
with high probability.

\item[Case 2:] 
For the case when type of $v_i \neq $ type of $v_j $ \\
From Lemma \ref{lm_low_bnd}(a), we get for any vertices $v$ and $w$ with high probability, 
\begin{align*}
\left|\left\{\{v, w\}: d_G(v, w) \leq (1 - \gve)\frac{\log n}{\log \gl}\right\}\right| \leq O(n^{2-\gve}).
\end{align*}
Also, from Lemma \ref{lm_up_bnd2}, we get
\begin{align*}
\Pr\left(d_G(v, w) < (1 + \gve)\frac{\log n}{\log \gl}\right) = 1 - exp(-\Omega(n^{2\eta}))
\end{align*}
So, putting the two statements together, we get that with high probability,
\begin{align*}
\sum_{i,j = 1: type(v_i)\neq type(v_j)}^n\left(\frac{\bld_{ij}}{\log n} - \D_{ij}\right)^2 =  O(n^{2-\gve}) + O(n^2).\gve^2 = O(n^{2-\gve}) 
\end{align*}
since, by Eq. \eqref{eq_gre} $\gre = O(1/\sqrt{n})$ and and $(1 - exp(-\Omega(n^{2\eta})))^{n^2} \rar 1$ as $n\rar \infty$.
\end{itemize}
So, putting the two cases together, we get that with high probability,
\begin{align*}
\sum_{i,j = 1}^n\left(\frac{\bld_{ij}}{\log n} - \D_{ij}\right)^2 =  O(n^{2-\gve}) + O(n^2).\gve^2 = O(n^{2-\gve}). 
\end{align*}
Hence,
\begin{align*}
\left|\left|\frac{\bld}{\log n} - \D\right|\right|_F \leq O(n^{1-\gve/2}).
\end{align*}

\subsection{Perturbation Theory of Linear Operators}
\label{sec_det_pert}
Once, we have the limiting behavior of the matrix $D$ established in Theorem \ref{thm_geo_dis}, we shall now try to see the behavior of the eigenvectors of the matrix $D$. Now, matrix $D$ can be considered as a perturbation of the operator $\D$.

The Davis-Kahan Theorem states a bound on perturbation of eigenspace instead of eigenvector, as discussed previously. The $\sin\gt$ Theorem of Davis-Kahan \cite{davis1970rotation} 
\begin{theorem}[Davis-Kahan (1970)\cite{davis1970rotation}]
\label{thm_dk}
Let $\bh, \bh' \in \R^{n\times n}$ be symmetric, suppose $\cV \subset \R$ is an interval, and suppose for some positive integer $d$ that $\bw, \bw' \in \R^{n\times d}$ are such that the columns of $\bw$ form an orthonormal basis for the sum of the eigenspaces of $\bh$ associated with the eigenvalues of $\bh$ in $\cV$ and that the columns of $\bw'$ form an orthonormal basis for the sum of the eigenspaces of $\bh'$ associated with the eigenvalues of $\bh'$ in $\cV$. Let $\gd$ be the minimum distance between any eigenvalue of $\bh$ in $\cV$ and any eigenvalue of $\bh$ not in $\cV$ . Then there exists an orthogonal matrix $\br \in \R^{d\times d}$ such that $||\bw\br - \bw'||_{F} \leq \sqrt{2}\frac{||\bh - \bh'||_{F}}{\gd}$.
\end{theorem}

\subsection{Proof of Theorem \ref{thm_mis}} 
Now, we can try to approximate limiting operator by the graph distance matrix $\bld$ in Frobenius norm based on Theorem \ref{thm_geo_dis} of Part I. The behavior of the eigenvalues of the limiting operator $\D$ can be stated as follows -
\begin{lemma}
The eigenvalues of $\D$ - $|\gl_1(\D)| \geq |\gl_2(\D)| \geq\cdots \geq|\gl_n(\D)|$, can be bounded as follows -
\begin{align}
\label{eq_eigvl_rel}
\gl_1(\D) < n,\ |\gl_K(\D)| > Cn,\ \gl_{Q+1}(\D) = -\min\{\tilde{d}_1,\ldots,\tilde{d}_Q\}, \ldots, \gl_{n} = -\max\{\tilde{d}_1,\ldots,\tilde{d}_Q\}
\end{align}
where, $\tilde{d}$, a vector of length $Q$, is defined in Eq. \eqref{eq_lim_rel} and the smallest $(n-Q)$ absolute eigenvalues of $\D$ are $-\tilde{d}$ where $-\tilde{d}_a$ has multiplicity $(n_a-1)$ for $a=1, \ldots, Q$.
\end{lemma} 
\begin{proof}
The matrix $\D$ can be considered as a Khatri-Rao product of the matrices $\cd$ and $\bj$ according to equation \eqref{eq_lim_rel}. Now, there exists a constant $\tau$ such that $\log||T_K||>\tau>0$, since $||T_K||>1$. So, we have $\gl_1(\cd) < \tau$. So, we have $\gl_1(\cd) < 1$ and since $n_a \leq n$ for all $a$ and $\sum_an_a = n$. So, we have $\gl_1(\D)\leq n$. Now, By Assumption (C2) and (C4), $\gl_Q(\cd) \geq \ga$ and $n_a \geq \gm n$, so, $\gl_Q(\D)\geq \ga\gm n$. Now, it is easy to see that the remaining eigenvalues of $\D$ is -1, since, $\cb\star\bj$ is a rank $Q$ matrix and its remaining eigenvalues are zero and the eigenvalues of diagonal matrix are $\tilde{d}$ with $\tilde{d}_a$ having multiplicity $(n_a)$ for $a=1, \ldots, Q$.
\end{proof}
\begin{corollary}
\label{cor_eig}
With high probability it holds that $|\gl_Q(\bld/\log n)| \geq O(n)$ and \\ 
$\gl_{Q+1}(\bld/\log n) \leq O(n^{1-\gve})$.
\end{corollary}
\begin{proof}
By Weyl's Inequality, for all $i=1, \ldots, n$,
\begin{eqnarray*}
||\gl_i(\bld/\log n)| - |\gl_i(\D)|| & \leq & \left|\left|\frac{\bld}{\log n} - \D\right|\right|_F \leq O(n^{1-\gve/2}) \\
 & \leq & O(n^{1-\gve})
\end{eqnarray*}
So, $|\gl_Q(\bld/\log n)| \geq O(n) - O(n^{1-\gve}) = O(n)$ for large $n$ and $|\gl_{Q+1}(\bld/\log n)| \leq -1 + O(n^{1-\gve}) = O(n^{1-\gve})$.
\end{proof}
Now, if we consider $\bw$ is the eigenspace corresponding to top $Q$ absolute eigenvalues of $\D$ and $\tilde{\bw}$ is the eigenspace corresponding to top $Q$ absolute eigenvalues of $\bld$. Using Davis-Kahan
\begin{lemma}
\label{lm_eigsp}
With high probability, there exists an orthogonal matrix $\br\in\R^{Q\times Q}$ such that $||\bw\br - \tilde{\bw}||_{F} \leq O(n^{-\gve})$
\end{lemma}
\begin{proof}
The top $Q$ eigenvalues of both $\D$ and $\bld/\log n$ lies in $(Cn, \infty)$ for some $C>0$. Also, the gap $\gd = O(n)$ between top $Q$ and $Q+1$th eigenvalues of matrix $\D$. So, now, we can apply Davis-Kahan Theorem \ref{thm_dk} and Theorem \ref{thm_geo_dis}, to get that,
\begin{align*}
||\bw\br - \tilde{\bw}||_F \leq \sqrt{2}\frac{\left|\left|\frac{\bld}{\log n} - \D\right|\right|_F}{\gd} \leq \frac{O(n^{1-\gve})}{O(n)} = O(n^{-\gve})
\end{align*} 
\end{proof}

Now, the relationship between the rows of $W$ can be specified based on Assumption (C3) as follows -
\begin{lemma}
\label{lm_rows}
For any two rows $i,j$ of $\bw_{n\times Q}$ matrix, $||u_i - u_j||_2 \geq O(1/\sqrt{n})$, if type of $v_i\neq $ type of $v_j$.
\end{lemma} 
\begin{proof}
The matrix $\D$ can be considered as a Khatri-Rao product of the matrices $\cd$ and $\bj$ according to equation \eqref{eq_lim_rel}. Now, by Assumption (C3), we have a constant difference between the rows of matrix $\cd$. So, rows of $\D$ as well as the projection of $\D$ into into its top $Q$ eigenspace has difference of order $O(n^{-1/2})$ between rows of matrix. 
\end{proof}

Now, if we consider $Q$-means criterion as the clustering criterion on $\tilde{\bw}$, then, for the $Q$-means minimizer centroid matrix $\bc$ is an $n\times Q$ matrix with $Q$ distinct rows corresponding to the $Q$ centroids of $Q$-means algorithm. By property of $Q$-means objective function and Lemma \ref{lm_eigsp}, with high probability,
\begin{eqnarray*}
||\bc - \tilde{\bw}||_F & \leq & ||\bw\br - \tilde{\bw}||_F\\
||\bc - \bw\br||_F & \leq & ||\bc - \tilde{\bw}||_F + ||\bw\br - \tilde{\bw}||_F \\
 & \leq & 2||\bw\br - \tilde{\bw}||_F \\
 & \leq & O(n^{-\gve}) \\
\end{eqnarray*}

By Lemma \ref{lm_rows}, for large $n$, we can get constant $C$, such that, $Q$ balls, $B_1, \ldots, B_Q$, of radius $r = Cn^{-1/2}$ around $Q$ distinct rows of $\bw$ are disjoint.

Now note that with high probability the number of rows $i$ such that $||\bc_i - (\bw\br)_i|| > r$ is at most $O(n^{1/2-\gve})$. If the statement does not hold then,
\begin{eqnarray*}
||\bc-\bw\br||_F & > & r.O(n^{1/2-\gve}) \\
 & \geq & Cn^{-1/2}.O(n^{1/2-\gve}) = O(n^{-\gve})
\end{eqnarray*}
So, we get a contradiction, since $||\bc - \bw\br||_F \leq O(n^{-\gve})$. Thus, the number of mistakes should be at most of order $O(n^{1/2-\gve})$. 

So, for each $v_i \in V(G_n)$, if $\xi_i$ is the type of $v_i$ and $\hat{\xi}_i$ is the type of $v_i$ as estimated from applying $Q$-means on top $Q$ eigenspace of geodesic matrix $\bld$, we get that with high probability, for some small $0<\gve$,
\begin{align*}
\min_{\pi\in\mathcal{P}_Q} |\{u \in V: \xi(u)\neq \pi(\hat{\xi}(u))\}| = O(n^{1/2-\gve})
\end{align*}

\section{Application}
\label{sec_geo_appl}
We investigate the empirical performance of the algorithm in several different setup. At first, we use simulated networks from stochastic block model to find the empirical performance of the algorithm. Then, we apply our method to find communities in several real world networks.

\subsection{Simulation}
We simulate networks from stochastic block models with $Q=3$ blocks. Let $w$ correspond to a $Q$-block model defined by parameters $\gt = (\V{\pi}, \rho_n, S)$, where $\pi_a$ is the probability of a node being assigned to block $a$ as before, and
\begin{align*}
\mathbf{F}_{ab} = \Pr(A_{ij} =1|i\in a,j\in b)=\rho_nS_{ab},\ \ \  1\leq a,b \leq K.
\end{align*}
and the probability of node $i$ to be assigned to block $a$ to be $\pi_a$ ($a=1, \ldots, K$).

\subsubsection{Equal Density Clusters}
We consider a stochastic block model with $Q = 3$. We consider the parameter matrix $\mathbf{F} = 0.012(1+0.1\nu)(\til{\gl}F^{(1)} + (1-\til{\gl})F^{(2)})$, where, $F^{(1)}_{3\times 3} =  \mbox{Diag}(0.9, 0.9, 0.9)$ and $F^{(2)}_{3\times 3} = 0.1\mathbf{J}_2$, where, $\mathbf{J}_2$ is a $2\times 2$ matrix of all 1's and $\nu$ varies from 1 to 15 to give networks of different density. So, we get $\rho_n = \V{\pi}^T\mathbf{F}\V{\pi}$. We now, vary $\til{\gl}$ to get different combinations of $\mathbf{F}$ as well as $\rho_n$. 

In the following figures, we try to see the behavior of mean and variances of the count statistics, as we vary $\gl_n$ as we vary $\nu$.
\begin{figure}
\centering
  \includegraphics[height=4.5cm,width=5cm]{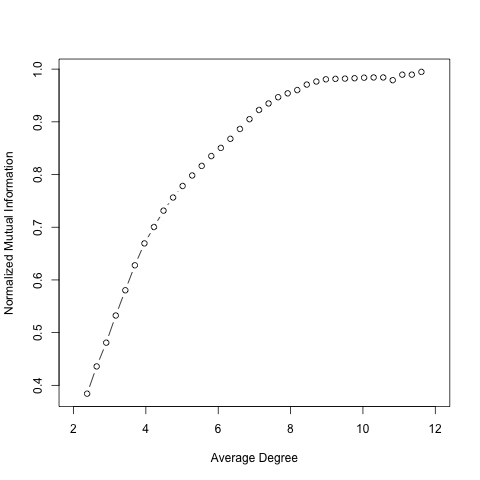}\hfill
  \includegraphics[height=4.5cm,width=5cm]{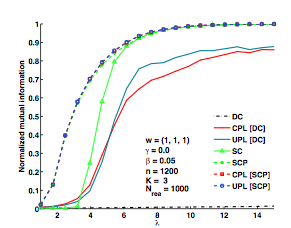}
 \caption{The LHS is the performance of graph distance based method and RHS is the performance of Pseudolikelihood method on same generative SBM.}
 \label{fig_eq}
 \end{figure}

\subsubsection{Unequal Density Clusters}
We consider a stochastic block model with $Q = 3$. We consider the parameter matrix $\mathbf{F} = 0.012(1+0.1\nu)(\til{\gl}F^{(1)} + (1-\til{\gl})F^{(2)})$, where, $F^{(1)}_{3\times 3} =  \mbox{Diag}(0.1, 0.5, 0.9)$ and $F^{(2)}_{3\times 3} = 0.1\mathbf{J}_2$, where, $\mathbf{J}_2$ is a $2\times 2$ matrix of all 1's and $\nu$ varies from 1 to 15 to give networks of different density. So, we get $\rho_n = \V{\pi}^T\mathbf{F}\V{\pi}$. We now, vary $\til{\gl}$ to get different combinations of $\mathbf{F}$ as well as $\rho_n$. 

In the following figures, we try to see the behavior of mean and variances of the count statistics, as we vary $\gl_n$ as we vary $\nu$.
\begin{figure}
\centering
  \includegraphics[height=4.5cm,width=5cm]{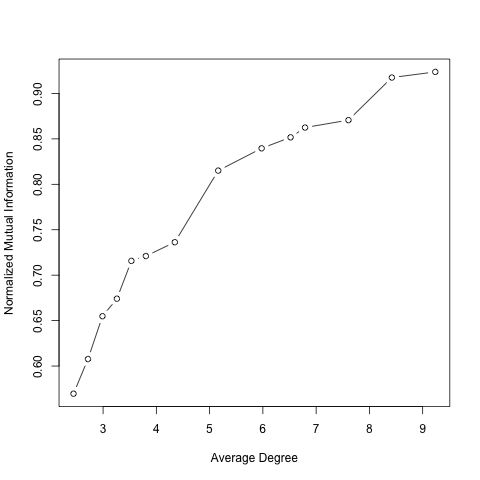}\hfill
  \includegraphics[height=4.5cm,width=5cm]{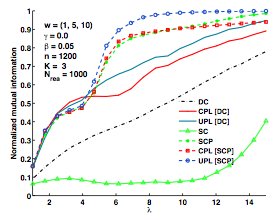}
 \caption{The LHS is the performance of graph distance based method and RHS is the performance of Pseudolikelihood method on same generative SBM.}
 \label{fig_uneq}
 \end{figure}

\subsection{Application to Real Network Data}
\subsubsection{Facebook Collegiate Network}
In this application, we try to find communities for Facebook collegiate networks. The networks were presented in the paper by Traud et.al. (2011) \cite{traud2011comparing}. The network is formed by Facebook users acting as nodes and if two Facebook users are ``friends" there is an edge between the corresponding nodes. Along with the network structure, we also have the data on covariates of the nodes. Each node has covariates: gender, class year, and data fields that represent (using anonymous numerical identifiers) high school, major, and dormitory residence. We consider the network of a specific college (Caltech). We compare the communities found with the dormitory affiliation of the nodes.
\begin{figure}
\centering
  \includegraphics[height=4.5cm,width=5cm]{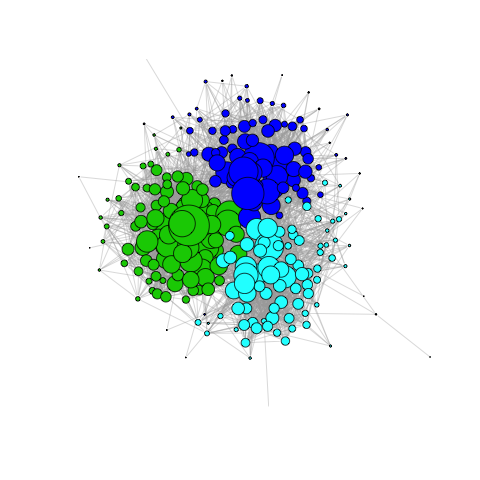}\hfill
  \includegraphics[height=4.5cm,width=5cm]{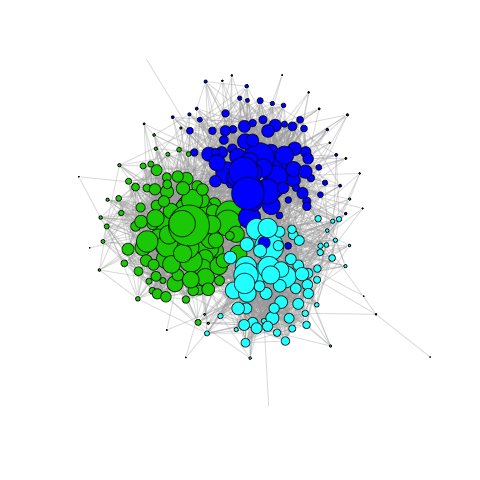}
 \caption{The LHS is community allocation and RHS is the one estimated by graph distance for Facebook Caltech network with 3 dorms.}
 \label{fig_fb}
 \end{figure}
 
\subsubsection{Political Web Blogs Network}
This dataset on political blogs was compiled by \cite{adamic2005political} soon after the 2004 U.S. presidential election. The nodes are blogs focused on US politics and the edges are hyperlinks between these blogs. Each blog was manually labeled as liberal or conservative by \cite{adamic2005political}, and we treat these as true community labels. We ignore directions of the hyperlinks and analyze the largest connected component of this network, which has 1222 nodes and the average degree of 27. The distribution of degrees is highly skewed to the right (the median degree is 13, and the maximum is 351). This is a network where the degree distribution is heavy-tailed and the graph is inhomogeneous.
\begin{figure}
\centering
  \includegraphics[height=4.5cm,width=5cm]{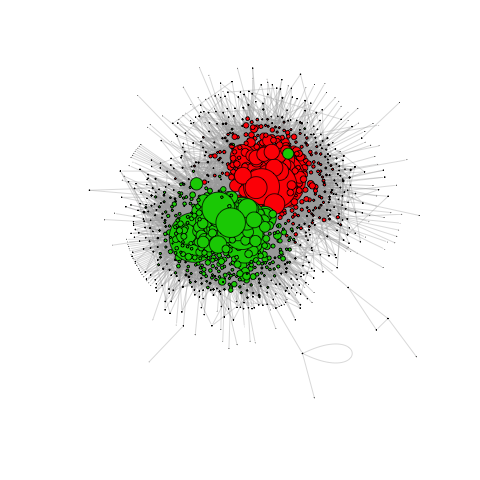}\hfill
  \includegraphics[height=4.5cm,width=5cm]{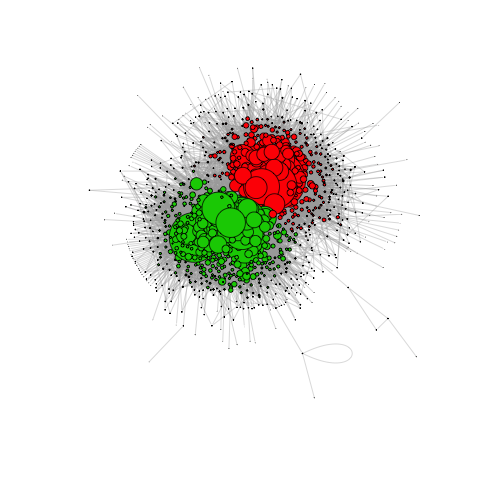}
 \caption{The LHS is community allocation and RHS is the one estimated by graph distance for Political Web blogs Network.}
 \label{fig_pwb}
 \end{figure}

\section{Conclusion}
\label{sec_geo_con}

The proposed graph distance based community detection algorithm gives a very general way for community detection for graphs over a large range of densities - from very sparse graphs to very dense graphs. We theoretically prove the efficacy of the method under the model that the graph is generated from stochastic block model with fixed number of blocks. We prove that the proportion of mislabeled communities goes to zero as the number of vertices $n \rar \infty$. This result is true for graphs coming from stochastic block model under certain conditions on the stochastic block model parameters. These conditions are satisfied above the threshold of block identification for two blocks as given in \cite{mossel2012stochastic}. The condition (C1) of $\mathbf{1}$ not being the eigenvector of $\tilde{K}$ for our community identification result to hold, seems to be an artificial one, as simulation suggests that our method is able to identify communities, even when $\mathbf{1}$ is an eigenvector of $\tilde{K}$.

We demonstrate the empirical performance of the method by using both simulated and real world networks. We compare with the pseudo-likelihood method and show that they have similar empirical performances. We demonstrate the empirical performance by applying the method for community detection in several real world networks too.

The method also works when number of blocks in the stochastic block model brows with $n$ (number of vertices) and for degree-corrected block model \cite{karrer2011stochastic}. We conjecture that under these models too the method will have the theoretical guarantee of correct community detection. The proof can be obtained by using similar techniques that we have used in this paper.

\bibliographystyle{plain}
\bibliography{network_geodesic_paper2}

\section*{Appendix: Branching Process Results}
\subsection*{\textbf{A1.} Proof of Lemma \ref{lemma_brpr_sbm}}

We have $n_a$ vertices of type $a$ ,$a = 1, \ldots,Q$, and that $n_a/n \stackrel{a.s.}{\rar} \pi_a$. From now on we condition on $n_1,\ldots,n_Q$; we may thus assume that $n_1,\ldots,n_Q$ are deterministic with $n_a/n\rar \pi_a$.
Let $\omega(n)$ be any function such that $\omega(n) \rar \infty$ and $\omega(n)/n \rar 0$. We call a component of $G_n \equiv G(n,P) = G(n,K/n)$ big if it has at least $\om(n)$ vertices. Let $B$ be the union of the big components, so $|B| = N_{\geq \om(n)}(G_n)$.
Fix $\gre>0$.We may assume that $n$ is so large that $\om(n)/n<\gre$ $\pi_i$ and $|n_a/n - \pi_a|<\gre$ $\pi_a$ for every $a$; thus $(1- \gre)\pi_a n<n_a <(1+\gre)\pi_an$. We may also assume that $n>\max K$, as $K$ is a function on the finite set $\s \times \s$. Since, $n_a/n$ is a $\sqrt{n}$-consistent estimator of $\pi_a$, we get that 
\begin{align}
\label{eq_gre}
\gre = O(n^{-1/2}).
\end{align}

Select a vertex and explore its component in the usual way, that means looking at its neighbors, one vertex at a time. We first reveal all edges from the initial vertex, and put all neighbors that we find in a list of unexplored vertices; we then choose one of these and reveal its entire neighborhood, and so on. Stop when we have found at least $\om(n)$ vertices (so $x \in B$), or when there are no unexplored vertices left (so we have found the entire component and $x \notin B$).

Consider one step in this exploration, and assume that we are about to reveal the neighborhood of a vertex $x$ of type $a$. Let us write $n'_b$ for the number of unused vertices of type $b$ remaining. Note that $n_b \geq n'_b \geq n_b - \om(n)$, so
\begin{align}
\label{eq_prop_bound}
(1 - 2\gre)\pi_b < n'_b/n < (1 + \gre)\pi_b
\end{align}
The number of new neighbors of $x$ of type $b$ has a binomial $Bin(n'_b, K_{ab}/n)$ distribution, and the numbers for different $b$ are independent. The total variation distance between a binomial $Bin(n, p)$ distribution and the Poisson distribution with the same mean is at most p. Hence the total variation distance between the binomial distribution above and the Poisson distribution $Poi(K_{ab}n'_b/n)$ is at most $K_{ab}/n = O(1/n)$. Also, by \eqref{eq_prop_bound},
\begin{align}
\label{eq_br_sbm}
(1 - 2\gre)K_{ab}\pi_b \leq K_{ab}n'_b/n \leq (1+ \gre) K_{ab}\pi_b.
\end{align}
Since we perform at most $\om(n)$ steps in the exploration, we may, with an error probability of $O(\om(n)/n) = o(1)$, couple the exploration with two multi-type branching processes $\cb((1-2\gre)K)$ and $\cb((1+\gre)K)$ such that the first process always finds at most as many new vertices of each type as the exploration, and the second process finds at least as many. Consequently, for a vertex $x$ of type $a$,
\begin{align}
\label{eq_brpr_bound}
\rho_{\geq \om(n)}((1 - 2\gre)K; a) + o(1) \leq \Pr(x \in B) \leq \rho_{\geq \om(n)}((1 + \gre)K; a) + o(1).
\end{align}

Since $\om(n) \rar \infty$, by Lemma 9.5 of \cite{bollobas2007phase}, we have $\rho_{\geq \om(n)}(K; a) \rar \rho(K; a)$ for every matrix or finitary kernel $K$, which parametrizes the offspring distribution of the branching process in the sense that the number of offsprings of type $b$ coming from a parent of type $a$ follows $Poi(K_{ab}\pi_b)$ distribution. So we can rewrite \eqref{eq_brpr_bound} as
\begin{align}
\label{eq_brpr_bound1}
\rho((1 - 2\gre)K; a) + o(1) \leq P(x \in B) \leq \rho((1 + \gre)K; a) + o(1).
\end{align}

\subsection*{\textbf{A2.} proof of Lemma \ref{lemma_brpr_cond}}

We need to consider certain branching process expectations $\gs(K)$ and $\gs_{\geq k}(K)$ in place of $\rho(K)$ and $\rho_{\geq k}(K)$. In preparation for the proof, we shall relate $\zeta(K)$ to the branching process $\cb_K$ via $\gs(K)$. As before, we assume that $K$ is a kernel on $(\s,\pi)$ with $K \in L^1$.

Let $A$ be a Poisson process on $\s$, with intensity given by a finite measure $\gl$, so that $A$ is a random multi-set on $\s$. If $g$ is a bounded measurable function on multi-sets on $\s$, it is easy to see that
\begin{align}
\label{eq_po_id}
\E(|A|g(A)) = \sum_{i \in \s} \E g(A\cup\{i\})\gl_i
\end{align}
For details see Proposition 10.4 of \cite{bollobas2007phase}.

Let $B(x)$ denote the first generation of the branching process $\cb_K (x)$. Thus $B(x)$ is given
by a Poisson process on $\s$ with intensity $K(x, y) \pi_x$. Suppose that $\sum_{b} K_{ab}\pi_b < \infty$ for every $a = 1, \ldots, Q$, so $B(x)$ is finite. Let $\gs(K;x)$ denote the expectation of $|B(x)|\mathbf{1}[|\cb_K(x)| = \infty]$, recalling that under the assumption $\sum_{b} K_{ab}\pi_b < \infty$ for every $a$, the branching process $\cb_K (x)$ dies out if and only if $|\cb_K (x)| < \infty$. Then
\begin{eqnarray*}
\sum_{b = 1}^Q K_{xb}\pi_b - \gs(K;x) & = & \E\left[|B(x)|\mathbf{1}(\cb_K(x) < \infty)\right] \\
 & = & \E\left(|B(x)|\prod_{z \in B(x)} \rho(K; z)\right) \\
 & = & \sum_{b = 1}^Q K_{xb}(1 - \rho(K;b))\E\left(\prod_{z \in B(x)} \rho(K; z)\right)\pi_b \\
 & = & \sum_{b = 1}^Q K_{xb}(1 - \rho(K;b))(1 - \rho(K;x)) \pi_b
\end{eqnarray*}
Here the penultimate step is from \eqref{eq_po_id}; the last step uses the fact that the branching process dies out if and only if none of the children of the initial particle survives. Writing $B$ for the first generation of $\cb_K$ conditioned on survival becomes
\begin{align*}
\gs(K) \equiv \E |B|\mathbf{1}[|\cb_K| = \infty] = \sum_{x = 1}^Q \gs(K; x) \pi_x
\end{align*}
Then, integrating over $x$ and subtracting from $\sum_{a,b} K_{ab}\pi_a\pi_b$, we get, 
\begin{align}
\label{eq_int_cond_kernel}
\gs (K) = \sum_{a,b} K_{ab}\left(1 - (1 - \rho(K;a))(1 - \rho(K;b))\right)\pi_a\pi_b
\end{align}
So, the kernel for the conditioned branching process becomes
\begin{align}
\label{eq_cond_kernel}
K_{ab}\left(\rho(K;a) + \rho(K;b) - \rho(K; a)\rho(K; b)\right)
\end{align}

\subsection*{\textbf{A3.} proof of Lemma \ref{lm_low_bnd}}

We have $\s$ is finite, say $\s = \{1,2,\ldots,Q\}$. Let $\G_d(v) \equiv \G_d(v, G_n)$ denote the $d$-distance set of $v$ in $G_n$, i.e., the set of vertices of $G_n$ at graph distance exactly $d$ from $v$, and let $\G_{\leq d}(v) \equiv \G_{\leq d}(v,G_n)$ denote the $d$-neighborhood $\cup_{d'\leq d}\G_{d'}(v)$ of $v$.

Let $0 < \gve < 1/10$ be arbitrary. The proof of \eqref{eq_brpr_bound1} involved first showing that, for $n$ large enough, the neighborhood exploration process starting at a given vertex $v$ of $G_n$ with type $a$ (chosen without inspecting $G_n$) could be coupled with the branching process $\cb_{(1+ \gve)K'} (i)$, where the $K'$ is defined by equation \eqref{eq_cond_kernel}, so that the branching process is conditioned to survive. However, henceforth we shall abuse notation and denote $K'$ as $K$. 

The neighborhood exploration process and multi-type branching process can be coupled so that for every $d$, $|\G_{d}(v)|$ is at most the number $N_d$ of particles in generation $d$ of $\cb_{(1+2\gve)K}(i)$. The number of vertices at generation $d$ of type $c$ of branching process $\cb_{(1+2\gve)K}(a)$, denoted by $N^a_{d,c}$ and the number of vertices of type $c$ at distance $d$ from $v$ for the neighborhood exploration process of $G_n$ is denoted by $|\G^a_{d, c}(v)|$f, where, $c=1, \ldots, Q$.

Elementary properties of the branching process imply that $\E N_d = O\left( ||T_{(1+2\gve)K}||^d\right) = O(((1 + 2\gve)\gl)^d)$, where $\gl = ||T_{K}|| > 1$.

Let $N^a_t(c)$ be the number of particles of type $c$ in the $t$-th generation of $\cb_K(a)$, then, $N^a_t$ is the vector $(N^a_t(1), \ldots, N^a_t (Q))$. Also, let $\nu = (\nu_1, \ldots, \nu_Q)$ be the eigenvector of $T_K$ with eigenvalue $\gl$ (unique, up to normalization, as $P$ is irreducible). From standard branching process results, we have
\begin{align}
\label{eq_brpr_lim1}
N^a_t/\gl^t \stackrel{a.s.}{\rar} X\nu,
\end{align}
where $X \geq 0$ is a real-valued random variable, $X$ is continuous except that it has some mass at 0, and $X = 0$ if and only if the branching process eventually dies out and lastly,
\begin{align*}
\E X = \nu_a.
\end{align*}
under the conditions given in Theorem V.6.1 and Theorem V.6.2 of \cite{athreya1972branching}.

Set $D = (1 - 10\gve)\log (n/\nu_a\nu_b)/\log \gl$. Then $D < (1 - \gve)\log (n/\nu_a\nu_b)/\log((1 + 2\gve)\gl)$ if $\gve$ is small enough, which we shall assume. Thus,
\begin{align*}
\E|\G_{\leq D}(v)| \leq \E\sum_{d = 0}^D N_d = O(((1 + 2\gve)\gl)^D) = O(n^{1 - \gve})
\end{align*}
So, summing over $v$, we have
\begin{align*}
\sum_{v\in V(G_n)}|\G_{\leq D}(v)| = \left|\left\{\{v, w\}: d_G(v, w) \leq (1 - \gve)\log(n/\nu_a\nu_b)/\log \gl\right\}\right| 
\end{align*}
and its expected value to be
\begin{align*}
\E\left|\left\{\{v, w\}: d_G(v, w) \leq (1 - \gve)\log (n/\nu_a\nu_b)/\log \gl\right\}\right| = \E\sum_{v\in V(G_n)}|\G_{\leq D}(v)| = O(n^{2-\gve})
\end{align*}
The above statement is equivalent to
\begin{align*}
\E\left|\left\{\{v, w\}: d_G(v, w) \leq (1 - \gve)\frac{\log n}{\log \gl/\log(\nu_a\nu_b)}\right\}\right| = \E\sum_{v\in V(G_n)}|\G_{\leq D}(v)| = O(n^{2-\gve})
\end{align*}
So, by Markov's Theorem, we have,
\begin{align*}
\Pr\left[\left|\left\{\{v, w\}: d_G(v, w) \leq (1 - \gve)\frac{\log n}{\log \gl/\log(\nu_a\nu_b)}\right\}\right| \leq O(n^{2-\gve/2}) \right] = o(1)
\end{align*}
for any fixed $\gve > 0$. 

\subsection*{\textbf{A4.} proof of Lemma \ref{lm_up_bnd1}}

We consider the branching process conditioned on survival. Now, we consider the single type branching process with offspring distribution $Poi(K_{aa}\pi_a)$ and the corresponding stochastic block model graph $G'_n$,  is the induced subgraph of the original graph $G_n$, where, vertices of $G'_n$ are only the vertices of type $a$ from $G_n$. So, $G'_n$ has in total $n_a$ vertices. So, we can always upper bound the the distance between two vertices of same type in $G_n$, by the distance between the two vertices in $G'_n$, since, the path representing distance between two vertices in $G'_n$ is present in $G_n$ but, the converse is not true. So, distance between two vertices in $G_n$ is always less than distance between two vertices in $G'_n$. So, we can say for any $v, w\in V(G_n) or V(G'_n)$ of same type, 
\begin{align*}
d_G(v, w) \leq d_{G'}(v,w)
\end{align*}
From here on, we shall abuse notation a bit and call $G'_n$ as $G_n$, since we are only considering the graph $G'_n$ from here on as the graph from stochastic block model.

We have $K_{aa} > 0$. Fix $0 < \eta < 1/10$. We shall assume that $\eta$ is small enough that $(1 - 2\eta)K_{aa}\pi_a > 1$. In the argument leading to \eqref{eq_brpr_bound1} in proof of Lemma \ref{lemma_brpr_sbm}, we showed that, given $\om(n)$ with $\om(n) = o(n)$ and a vertex $v$ of type $a$, the neighborhood exploration process of $v$ in $G_n$ could be coupled with the branching process $\cb_{(1-2\eta)K_{aa}}(a)$ so that whp the former dominates until it reaches size $\om(n)$. 

From here onwards we shall only consider a single-type branching process where particles have type $a$.  More precisely, writing $N_{d,a}$ for the number of vertices of type $a$ in generation $d$ of $\cb_{(1-2\eta)K_{aa}}(a)$, and $\G_{d,a}(v)$ for the set of type-$a$ vertices at graph distance $d$ from $v$, whp 
\begin{align}
\label{eq_brpr_count}
|\G_{d,a}(v)| \geq N_{d,a},\ \mbox{ for all } d \mbox{ s.t. } |\G_{\leq d}(v)| < \om(n).
\end{align}
This relation between the number of vertices at generation $d$ of branching process $\cb_{(1-2\eta)K_{aa}}(a)$, denoted by $N_{d,a}$ and the number of vertices at distance $d$ from $v$ for the neighborhood exploration process of $G_n$, denoted by $|\G_{d,a}(v)|$ becomes highly important later on in this proof. Note that the relation only holds when $|\G_{\leq d}(v)| < \om(n)$ for some $\om(n)$ such that $\om(n)/n \rar 0$ as $n\rar \infty$.

Now let us begin the second part of the proof. Let $N_t(a)$ be the number of particles of type $a$ in the $t$-th generation of $\cb_K$. Also, let $\gl_a = K_{aa}\pi_a$. From standard branching process results, we have
\begin{align}
\label{eq_brpr_lim}
N_t(a)/\gl_a^t \stackrel{a.s.}{\rar} X,
\end{align}
where $X \geq 0$ is a real-valued random variable, $X$ is continuous except that it has some mass at 0, and $X = 0$ if and only if the branching process eventually dies out.

Let $D$ be the integer part of $\log((n\pi_a)^{1/2+2\eta})/\log((1-2\eta)\gl_a)$.From \eqref{eq_brpr_lim}, conditioned on survival of branching process $\cb_{K_{aa}}(a)$, whp either $N_{D,a} = 0$ or $N_{D,a} \geq n^{1/2+\eta}$ (note that $N_{D,a}$ comes from branching process $\cb_{(1-2\eta)K_{aa}}(a)$ not branching process $\cb_{K_{aa}}(a)$). Furthermore, as $\lim_{d \rar \infty} \Pr(N_d \neq 0) = \rho((1 - 2\eta)K_{aa})$ and $D \rar \infty$, we have $\Pr(N_{D, a} \neq 0) \rar \rho((1 - 2\eta)K_{aa})$. Thus, if $n$ is large enough,
\begin{align*}
\Pr\left(N_{D,a} \geq (n\pi_a)^{1/2+\eta}\right) \geq \rho((1-2\eta)K_{aa}) - \eta.
\end{align*}
Now, we have conditioned that the branching process with kernel $K_{aa}$ is conditioned to survive. The right-hand side tends to $\rho(K_{aa}) = 1$ as $\eta \rar 0$. Hence, given any fixed $\gm > 0$, if we choose $\eta > 0$ small enough we have
\begin{align*}
\Pr\left(N_{D,a} \geq (n\pi_a)^{1/2+\eta}\right) \geq 1 - \gm 
\end{align*}
for $n$ large enough. 

Now, the neighborhood exploration process and branching process can be coupled so that for every $d$, $|\G_{d}(v)|$ is at most the number $M_d$ of particles in generation $d$ of $\cb_{(1+2\gve)K_{aa}}(a)$ from Lemma \ref{lemma_brpr_sbm} and Eq \eqref{eq_br_sbm}. So, we have, 
\begin{align*}
\E|\G_{\leq D}(v)| \leq \E\sum_{d = 0}^D M_d = O(((1 + 2\gve)\gl_a)^D) = o((n_a)^{2/3})
\end{align*}
if $\eta$ is small enough, since $D$ be the integer part of $\log((n\pi_a)^{1/2+2\eta})/\log((1-2\eta)\gl_a)$ and $|n_a/n - \pi_a| < \gve$. Note that the power $2/3$ here is arbitrary, we could have any power in the range $(1/2, 1)$. Hence,
\begin{align*}
|\G_{\leq D}(v)| \leq n_a^{2/3}\ \ whp,
\end{align*} 
and whp the coupling described in \eqref{eq_brpr_count} extends at least to the $D$-neighborhood. So, now, we are in a position to apply Eq \eqref{eq_brpr_count}, as we have $|\G_{\leq D}(v)| \leq n_a^{2/3} < \om(n)$, with $\om(n)/n \rar 0$.

Now let $v$ and $w$ be two fixed vertices of $G(n_a,P_a)$, of type $a$. We explore both their neighborhoods at the same time, stopping either when we reach distance $D$ in both neighborhoods, or we find an edge from one to the other, in which case $v$ and $w$ are within graph distance $2D + 1$. We consider two independent branching processes $\cb_{(1-2\eta)K_{aa}}(a)$, $\cb'_{(1-2\eta)K_{aa}}(a)$, with $N_{d, a}$ and $N'_{d, a}$ vertices of type $a$ in generation $d$ respectively. By previous equation, whp we encounter $o(n)$ vertices in the explorations so, by the argument leading to \eqref{eq_brpr_count}, whp either the explorations meet, or $|\G_{D, a}(v)| \geq N_{D, a}$ and $|\G_{D, a}(w)| \geq N'_{D, a}$ with the explorations not meeting. Using bound on $N_{d,a}$ and the independence of the branching processes, it follows that
\begin{align*}
\Pr\left(d(v,w) \leq 2D+1 \mbox{ or } |\G_{D,a}(v)|, |\G_{D,a}(w)| \geq n_a^{1/2+\eta}\right) \geq (1 - \gm)^2 - o(1). 
\end{align*}
Note that the two events in the above probability statement are not disjoint. We shall try to find the probability that the second event in the above equation holds but not the first. We have not examined any edges from $\G_D(v)$ to $\G_D(w)$, so these edges are present independently with their original unconditioned probabilities. The expected number of these edges is at least $|\G_{D,a}(v)||\G_{D,a}(w)|K_{aa}/n$. If $K_{aa} > 0$, this expectation is $\Omega((n^{1/2+\eta})^2/n) = \Omega(n^{2\eta})$. It follows that at least one edge is present with probability $1 - \exp(-\Omega(n^{2\eta})) = 1 - o(1)$. If such an edge is present, then $d(v, w) \leq 2D + 1$. So, the probability that the second event in the above equation holds but not the first is $o(1)$. Thus, the last equation implies that
\begin{align*}
\Pr(d(v,w) \leq 2D+1) \geq (1 - \gm)^2 - o(1) \geq 1 - 2\gm - o(1).
\end{align*}
Choosing $\eta$ small enough, we have $2D + 1 \leq (1 + \gve)\log n/\log \gl_a$. As $\gm$ is arbitrary, we
have
\begin{align*}
\Pr(d(v,w) \leq (1+\gve)\log n\pi_a/\log\gl_a) \geq 1 - \exp(-\Omega(n^{2\eta})).
\end{align*}
Now, $\gl_a = K_{aa}\pi_a$
and the lemma follows.

\subsection*{\textbf{A5.} proof of Lemma \ref{lm_up_bnd2}}

We consider the multi-type branching process with probability kernel $P_{ab}=\frac{K_{ab}}{n}$ $\forall a,b = 1, \ldots, Q$ and the corresponding random graph $G_n$ generated from stochastic block model has in total $n$ nodes. We condition that branching process $\cb_{K}$ survives.

Note that an upper bound $1$ is obvious, since we are bounding a probability, so it suffices to prove a corresponding lower bound. We may and shall assume that $K_{ab} > 0$ for some  $a,b$. 

Fix $0 < \eta < 1/10$. We shall assume that $\eta$ is small enough that $(1 - 2\eta)\gl > 1$. In the argument leading to \eqref{eq_brpr_bound1} in proof of Lemma \ref{lemma_brpr_sbm}, we showed that, given $\om(n)$ with $\om(n) = o(n)$ and a vertex $v$ of type $a$, the neighborhood exploration process of $v$ in $G_n$ could be coupled with the branching process $\cb_{(1-2\eta)K}(a)$ so that whp the former dominates until it reaches size $\om(n)$. More precisely, writing $N_{d,c}$ for the number of particles of type $c$ in generation $d$ of $\cb_{(1-2\eta)K}(a)$, and $\G_{d,c}(v)$ for the set of type $c$ vertices at graph distance $d$ from $v$, whp 
\begin{align}
\label{eq_brpr_count1}
|\G_{d,c}(v)| \geq N_{d,c},\ c = 1,\ldots,Q,\ \mbox{ for all } d \mbox{ s.t. } |\G_{\leq d}(v)| < \om(n).
\end{align}
This relation between the number of vertices at generation $d$ of type $c$ of branching process $\cb_{(1-2\eta)K}(a)$, denoted by $N_{d,c}$ and the number of vertices of type $c$ at distance $d$ from $v$ for the neighborhood exploration process of $G_n$, denoted by $|\G_{d, c}(v)|$ becomes highly important later on in this proof, where, $c=1, \ldots, Q$. Note that the relation only holds when $|\G_{\leq d}(v)| < \om(n)$ for some $\om(n)$ such that $\om(n)/n \rar 0$ as $n\rar \infty$.


Let $N^a_t(c)$ be the number of particles of type $c$ in the $t$-th generation of $\cb_K(a)$, then, $N^a_t$ is the vector $(N^a_t(1), \ldots, N^a_t (Q))$. Also, let $\nu = (\nu_1, \ldots, \nu_Q)$ be the eigenvector of $T_K$ with eigenvalue $\gl$ (unique, up to normalization, as $P$ is irreducible). From standard branching process results, we have
\begin{align}
\label{eq_brpr_lim1}
N^a_t/\gl^t \stackrel{a.s.}{\rar} X\nu,
\end{align}
where $X \geq 0$ is a real-valued random variable, $X$ is continuous except that it has some mass at 0, and $X = 0$ if and only if the branching process eventually dies out and lastly,
\begin{align*}
\E X = \nu_a
\end{align*}
under the conditions given in Theorem V.6.1 and Theorem V.6.2 of \cite{athreya1972branching}.

Let $D$ be the integer part of $\log((n)^{1/2+2\eta})/\log((1-2\eta)\gl)$. From \eqref{eq_brpr_lim1}, conditioned on survival of branching process $\cb_{K}(a)$, whp either $N^a_{D} =0$, or $N^a_{D,c} \geq n^{1/2+\eta}$ for each $c$ (note that $N^a_{D,c}$ comes from branching process $\cb_{(1-2\eta)K}(a)$ not branching process $\cb_{K}(a)$). Furthermore, as $\lim_{d \rar \infty} \Pr(N^a_d \neq 0) = \rho((1 - 2\eta)K)$ and $D \rar \infty$, we have $\Pr(N^a_D \neq 0) \rar \rho((1 - 2\eta)K)$. Thus, if $n$ is large enough,
\begin{align*}
\Pr\left(\forall c: N^a_{D,c} \geq n^{1/2+\eta}\right) \geq \rho((1-2\eta)K) - \eta.
\end{align*}
Now, we have conditioned that the branching process with kernel $K$ is conditioned to survive. The right-hand side tends to $\rho(K) = 1$ as $\eta \rar 0$. Hence, given any fixed $\gm > 0$, if we choose $\eta > 0$ small enough we have
\begin{align*}
\Pr\left(\forall c: N^a_{D,c} \geq n^{1/2+\eta}\right) \geq 1 - \gm 
\end{align*}
for $n$ large enough. 

Now, the neighborhood exploration process and branching process can be coupled so that for every $d$, $|\G_{d}(v)|$ is at most the number $M_d$ of particles in generation $d$ of $\cb_{(1+2\gve)K}(a)$ from Lemma \ref{lemma_brpr_sbm} and Eq \eqref{eq_br_sbm}. So, we have, 
\begin{align*}
\E|\G_{\leq D}(v)| \leq \E\sum_{d = 0}^D M_d = O(((1 + 2\gve)\gl)^D) = o(n^{2/3})
\end{align*}
if $\eta$ is small enough, since $D$ be the integer part of $\log(n^{1/2+2\eta})/\log((1-2\eta)\gl)$. Note that the power $2/3$ here is arbitrary, we could have any power in the range $(1/2, 1)$. Hence,
\begin{align*}
|\G_{\leq D}(v)| \leq n^{2/3}\ \ whp,
\end{align*} 
and whp the coupling described in \eqref{eq_brpr_count1} extends at least to the $D$-neighborhood. So, now, we are in a position to apply Eq \eqref{eq_brpr_count1}, as we have $|\G_{\leq D}(v)| \leq n_a^{2/3} < \om(n)$, with $\om(n)/n \rar 0$.

Now let $v$ and $w$ be two fixed vertices of $G(n,P)$, of types $a$ and $b$ respectively. We explore both their neighborhoods at the same time, stopping either when we reach distance $D$ in both neighborhoods, or we find an edge from one to the other, in which case $v$ and $w$ are within graph distance $2D + 1$. We consider two independent branching processes $\cb_{(1-2\eta)K}(a)$, $\cb'_{(1-2\eta)K}(b)$, with $N^a_{d, c}$ and $N^b_{d, c}$ vertices of type $c$ in generation $d$ respectively. By previous equation, whp we encounter $o(n)$ vertices in the explorations so, by the argument leading to \eqref{eq_brpr_count1}, whp either the explorations meet, or $|\G^a_{D, c}(v)| \geq N^a_{D, c}$ and $|\G^b_{D, c}(w)| \geq N^b_{D, c}$, $c = 1, \ldots, Q$, with the explorations not meeting. Using bound on $N^a_{d,c}$ and the independence of the branching processes, it follows that
\begin{align*}
\Pr\left(d(v,w) \leq 2D+1 \mbox{ or } \forall c :|\G^a_{D,c}(v)|, |\G^b_{D,c}(w)| \geq n^{1/2+\eta}\right) \geq (\rho(K) - \gm)^2 - o(1). 
\end{align*}
Note that the two events in the above probability statement are not disjoint. We shall try to find the probability that the second event in the above equation holds but not the first. We have not examined any edges from $\G_D(v)$ to $\G_D(w)$, so these edges are present independently with their original unconditioned probabilities. For any $c_1$, $c_2$, the expected number of these edges is at least $|\G^a_{D,c_1}(v)||\G^b_{D,c_2}(w)|K_{c_1c_2}/n$. Choosing $c_1, c_2$ such that $K_{c_1c_2} > 0$, this expectation is $\Omega((n^{1/2+\eta})^2/n) = \Omega(n^{2\eta})$. It follows that at least one edge is present with probability $1 - \exp(-\Omega(n^{2\eta})) = 1 - o(1)$. If such an edge is present, then $d(v, w) \leq 2D + 1$. So, the probability that the second event in the above equation holds but not the first is $o(1)$. Thus, the last equation implies that
\begin{align*}
\Pr(d(v,w) \leq 2D+1) \geq (1 - \gm)^2 - o(1) \geq 1 - 2\gm - o(1).
\end{align*}
Choosing $\eta$ small enough, we have $2D + 1 \leq (1 + \gve)\log (n)/\log \gl$. As $\gm$ is arbitrary, we
have
\begin{align*}
\Pr(d(v,w) \leq (1+\gve)\log (n)/\log\gl) \geq 1 - \exp(-\Omega(n^{2\eta})).
\end{align*}
The above statement is equivalent to
\begin{align*}
\Pr\left(d(v,w) \leq (1+\gve)\frac{\log n}{\log\gl}\right) \geq 1 - \exp(-\Omega(n^{2\eta})).
\end{align*}
and the lemma follows.

\end{document}